\documentclass[10pt]{article} %
\usepackage[accepted]{tmlr}

\usepackage{microtype}
\usepackage{hyperref}
\usepackage{url}
\usepackage{booktabs}
\usepackage{amsmath,amssymb,amsfonts,amsthm}
\usepackage{algorithm}
\usepackage{algorithmicx}
\usepackage{algpseudocodex}
\usepackage{graphicx}
\usepackage{float}
\usepackage{enumitem}
\usepackage{wrapfig}

\definecolor{darkblue}{rgb}{0, 0, 0.5}
\hypersetup{colorlinks=true, citecolor=darkblue, linkcolor=darkblue, urlcolor=darkblue}

\newtheorem{theorem}{Theorem}[section]
\newtheorem{definition}{Definition}[section]
\newtheorem{lemma}[theorem]{Lemma}
\newtheorem{corollary}[theorem]{Corollary}
\newtheorem*{lemma*}{Lemma}
\newtheorem*{theorem*}{Theorem}

\newtheorem{proposition}[theorem]{Proposition}

\DeclareMathOperator*{\argmax}{arg\,max}

\newcommand{\method}{SignCert-PO}
\newcommand{\methodlong}{Sign-Certified Policy Optimization}

\title{Mitigating Reward Hacking in Reinforcement Learning from Human Feedback via Advantage Sign Robustness}

\author{\name Shinnosuke Ono \email ono@ms.k.u-tokyo.ac.jp \\
      \addr The University of Tokyo
      \AND
      \name Johannes Ackermann \email ackermann@ms.k.u-tokyo.ac.jp \\
      \addr The University of Tokyo and RIKEN AIP
      \AND
      \name Soichiro Nishimori \email nishimori@ms.k.u-tokyo.ac.jp \\
      \addr The University of Tokyo and RIKEN AIP
      \AND
      \name Takashi Ishida \email ishi@k.u-tokyo.ac.jp \\
      \addr RIKEN AIP and The University of Tokyo
      \AND
      \name Masashi Sugiyama \email sugi@k.u-tokyo.ac.jp \\
      \addr RIKEN AIP and The University of Tokyo}

\def\month{10}
\def\year{2026}
\def\openreview{\url{https://openreview.net/forum?id=yu3lbJi5sh}}

\begin{document}

\maketitle

\begin{abstract}
	Reward models (RMs) used in reinforcement learning from human feedback (RLHF) are vulnerable to reward hacking: as the policy maximizes a learned proxy reward, true quality plateaus or degrades.
	We make the assumption that reward hacking is often caused by flipped advantage signs: instead of reducing the likelihood of a bad response, a flipped sign causes the update to increase it.
	By considering an adversarial perturbation in the RM parameter space, we can derive a \textit{certified sign-preservation radius}, which is the smallest perturbation that can flip the advantage sign during policy optimization.
	Based on this formulation, we propose \textit{\methodlong{} (\method{})}, down-weighting non-robust completions in the policy gradient update.
	Unlike prior approaches that require multiple RMs or access to the RM training data, \method{} is lightweight and operates purely at the policy optimization stage using only the RM parameters and on-policy completions.
	On TL;DR summarization and AlpacaFarm benchmarks, \method{} achieves a better win rate than baselines in most cases and reduces reward hacking.
\end{abstract}
\section{Introduction}

Reinforcement learning from human feedback (RLHF) has become a standard paradigm for aligning large language models (LLMs) with human intent~\citep{christiano_deep_2017, stiennon_learning_2020, ouyang_training_2022}. In the typical RLHF pipeline, a reward model (RM) is trained on pairwise human preference data and then used as a proxy for the true, unobservable preference signal during policy optimization.
Because the RM is inherently an imperfect proxy, optimizing the policy against it often causes the policy to exploit local inaccuracies in the reward landscape. This phenomenon is known as reward hacking~\citep{gao_scaling_2023}. As optimization progresses, the proxy reward continues to rise while true quality stagnates or degrades. This is especially acute when the proxy RM is trained with limited preference data~\citep{gao_scaling_2023} or when the policy drifts from the RM's training distribution during policy optimization by reinforcement learning (RL)~\citep{ackermann_off-policy_2025}.

Prior approaches to mitigating reward hacking can be organized along an axis from \emph{ante-hoc} to \emph{post-hoc}~\citep{retzlaff_post-hoc_2024}. Ante-hoc methods incorporate uncertainty quantification into the reward modeling pipeline. For instance, ensemble-based methods~\citep{coste_reward_2023, eisenstein_helping_2024} train multiple RMs and optimize conservatively with respect to their disagreement, while distributional or uncertainty-aware RMs~\citep{xu_uncertainty-aware_2024} modify the RM architecture or training objective. These approaches can directly quantify RM uncertainty, but require either multiple RM training and inference runs or changes to the RM pipeline, limiting their practical applicability. In contrast, in the post-hoc approach, the RM is left unchanged and robustness is introduced only at the policy optimization stage. Even AdvPO~\citep{zhang_mitigating_2024}, the closest prior work to ours, is only partially post-hoc, as its uncertainty estimation requires access to the RM training dataset.
Our approach is \emph{entirely post-hoc}: it operates solely with the current RM parameters and on-policy completions, and thus we do not need to train new RMs nor retain the dataset used for RM training. We discuss additional related work in \S\ref{sec:related-work}.

We argue that what matters during policy optimization is whether the RM correctly predicts the \textit{sign} of each completion's advantage, since that sign determines whether the policy increases or decreases the completion's probability. Crucially, the reliability of this prediction varies per completion, and a naive robustness penalty fails to capture this variability.
We therefore ask: for a given completion, how much would the RM need to change before that advantage sign flips? Based on randomized smoothing~\citep{cohen_certified_2019}, we formalize this quantity as a \textit{certified sign-preservation radius} and down-weight completions whose advantage signs are easy to flip (Figure~\ref{fig:overview}). 
We first show that directly considering perturbations of all RM parameters is computationally challenging, as it requires per-completion gradients.
This quickly causes us to run out of memory, even for relatively small RMs.
Instead, we propose to approximate this solution by only considering perturbations of the RM head, yielding a lightweight re-weighting of the standard policy gradient update. The resulting method, \textit{\methodlong{} (\method{})}, prevents policy updates from being dominated by directions whose update signs are locally fragile under RM perturbations.

We evaluate \method{} on TL;DR summarization~\citep{stiennon_learning_2020} and AlpacaFarm~\citep{dubois_alpacafarm_2023} with Pythia~\citep{biderman_pythia_2023} and Qwen2.5~\citep{qwen_qwen25_2025} models. Using the gold-model evaluation setup~\citep{gao_scaling_2023}, we compare against Dr.GRPO~\citep{liu_understanding_2025}, uncertainty-weighted optimization (UWO) with RM ensembles~\citep{coste_reward_2023}, behavior-supported policy optimization (BSPO)~\citep{dai_mitigating_2024}, and AdvPO~\citep{zhang_mitigating_2024}. \method{} achieves the highest gold-RM win rate in most settings, with improved RM accuracy during policy optimization.

\begin{figure}[t]
	\centering
	\includegraphics[width=\textwidth]{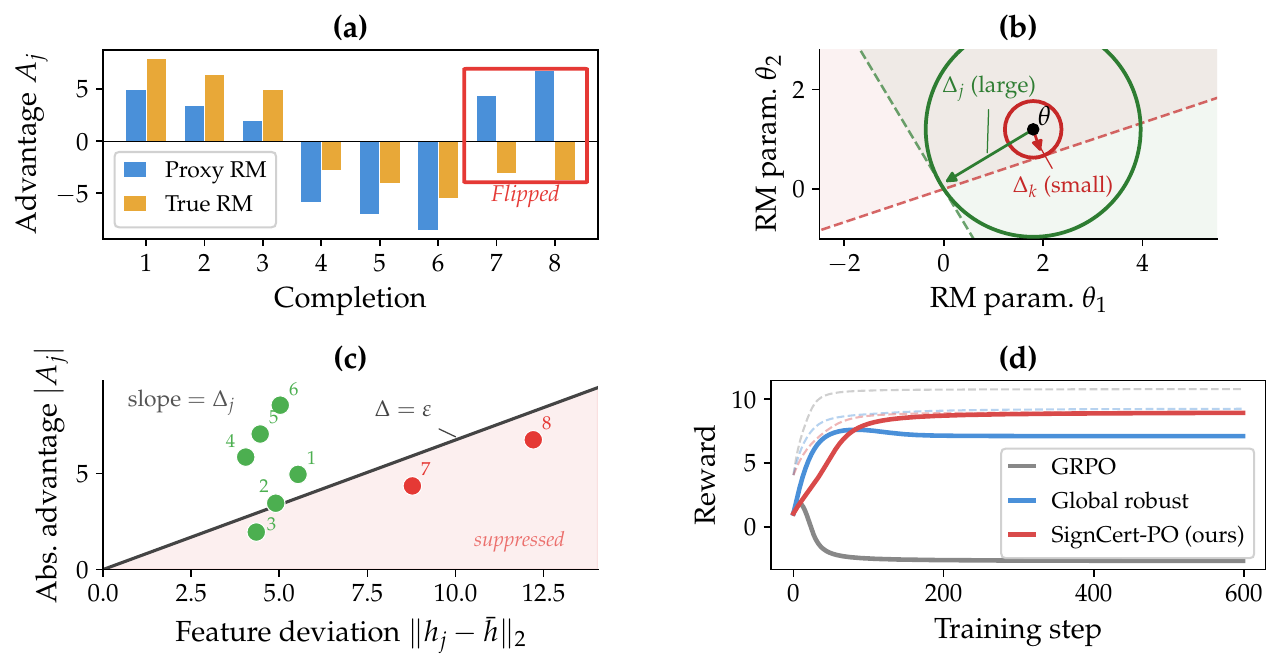}%
	\caption{\textbf{We argue that the reliability of the proxy RM's estimates differs by completion. Certified sign-preservation radius $\Delta_j$ provides this reliability measure.}
		\textbf{(a)}~Proxy and true advantages. Completions 7 and 8 have opposite signs, showing the proxy RM is unreliable there.
		\textbf{(b)}~$\Delta_j$ is the smallest perturbation of the RM parameters that flips a completion's advantage sign. Dashed lines are decision boundaries for the $j$- and $k$-th completions.
		\textbf{(c)}~Per-completion $\Delta_j.$ Completions 7 and 8 exhibit low $\Delta$, confirming $\Delta_j$ identifies unreliable completions. \method{} down-weights completions below $\Delta=\epsilon$.
	\textbf{(d)}~True (solid) and proxy (dashed) reward during RL. Using $\Delta_j$ as a re-weighting mechanism, \method{} prevents reward hacking and further improves the true reward. See Appendix~\ref{app:toy-bandit} for details.}
	\label{fig:overview}
\end{figure}

\section{Background}

We denote the set of prompts by $\mathcal{X}$ and the set of responses by $\mathcal{Y}.$ We assume a prompt distribution $P$ over $\mathcal{X}$, where $P(x)$ denotes the probability mass of a prompt $x\in\mathcal{X}.$ An autoregressive language model $\pi_\phi$ with parameters $\phi$ is then conditional probability $\pi_\phi(y\vert x)=\Pi_t\pi_\phi(y_t\vert x,y_{<t}),$ where $x\in\mathcal{X}$ is a prompt, $y\in\mathcal{Y}$ is the response, $y_t$ is the $t$-th response token and $y_{<t}$ are the response tokens before $y_t$. Following the RL literature~\citep{sutton_reinforcement_2018}, we sometimes use the word \textit{policy,} referring to the language model $\pi_\phi$. We denote the $\ell_2$ norm by $\Vert\cdot\Vert_2$.

\subsection{Reinforcement learning from human feedback}

RLHF typically consists of three key steps: supervised fine-tuning (SFT), reward modeling, and reinforcement learning~\citep{stiennon_learning_2020}. We denote the SFT model as $\pi_\text{SFT}.$

\paragraph{Reward modeling.}
Suppose the true reward function $r^*: \mathcal{X}\times\mathcal{Y}\rightarrow\mathbb{R}$ that mirrors human preference exists. Since $r^*$ is generally unknown, an RM is trained as a proxy with the following procedure: given a prompt $x\in\mathcal{X},$ two responses $y_0, y_1\in\mathcal{Y}$ are sampled from $\pi_\text{SFT}(y\vert x)$. A human annotator then selects their preferred response, yielding a dataset of triples $(x, y_\mathrm{w}, y_\mathrm{l})$ where $y_\mathrm{w}$ (winning) is preferred over $y_\mathrm{l}$ (losing). A scalar RM $r_\theta: \mathcal{X}\times\mathcal{Y}\rightarrow\mathbb{R}$ is trained under the Bradley-Terry (BT) model~\citep{bradley_rank_1952} by minimizing the cross-entropy loss:
\begin{equation}
	\mathcal{L}_\mathrm{BT}(\theta) = -\mathbb{E}_{(x, y_\mathrm{w}, y_\mathrm{l})}\left[\log \sigma\bigl(r_\theta(x,y_\mathrm{w}) - r_\theta(x,y_\mathrm{l})\bigr)\right],
	\label{eq:bt-loss}
\end{equation}
where $\sigma: \mathbb{R}\rightarrow\mathbb{R}$ is the sigmoid function: $\sigma(z) := 1/(1+\exp(-z)).$

\paragraph{Reinforcement learning.}
The policy $\pi_\phi$ is optimized to maximize the Kullback-Leibler (KL)-regularized expected reward~\citep{ziegler_fine-tuning_2020}:
\begin{equation}
	J(\phi, \theta) = \mathbb{E}_{x\sim P(x), y\sim\pi_\phi}\bigl[r_\theta(x, y) - \beta D_\mathrm{KL}(\pi_\phi(\cdot|x) \| \pi_\text{ref}(\cdot | x))\bigr],
	\label{eq:rlhf-objective}
\end{equation}
where $D_\mathrm{KL}(p\Vert q)$ is the KL divergence of probability densities $p$ from $q,$ and $\beta \ge 0$ controls the strength of KL regularization against a reference policy $\pi_\text{ref}$. This regularization discourages the policy from diverging too far from the RM's training distribution~\citep{gao_scaling_2023}. We denote this regularized reward by $r'_\theta(x,y):=r_\theta(x,y)-\beta D_\mathrm{KL}(\pi_\phi(\cdot\vert x)\Vert\pi_\mathrm{ref}(\cdot\vert x)).$ Policy optimization proceeds via the policy gradient (PG)~\citep{sutton_policy_1999}:
\begin{equation}
	\nabla_\phi J(\phi, \theta) = \mathbb{E}_{x \sim P, y \sim \pi_\phi(\cdot|x)} [A^{\pi_\phi}_\theta(x,y) \nabla_\phi \log \pi_\phi(y|x)],
	\label{eq:policy-gradient}
\end{equation}
where $A_\theta^{\pi_\phi}(x,y) := r'_\theta(x,y) - V_\theta^{\pi_\phi}(x)$ is the advantage function with value function $V_\theta^{\pi_\phi}(x) := \mathbb{E}_{y\sim\pi_\phi}[r'_\theta(x,y)].$ Different PG-based algorithms for RLHF differ mainly in their update rules and advantage estimation. In this work, we adopt Dr.GRPO~\citep{liu_understanding_2025}\label{sec:grpo}, a widely-used variant of GRPO~\citep{shao_deepseekmath_2024}. Given a prompt $x\in\mathcal{X}$, we sample $K$ completions $y^{(1)}, \ldots, y^{(K)}\in\mathcal{Y}$ from the current policy $\pi_\phi(\cdot|x).$ Each receives a reward $r_\theta(x, y^{(j)})$ for $j\in[K]:=\{1,\ldots,K\}.$ Dr.GRPO then provides the following advantage estimation:
\begin{equation}
	A_j(\theta) := r_\theta(x,y^{(j)}) - \frac{1}{K}\sum_{k=1}^K r_\theta(x,y^{(k)}).
	\label{eq:drgrpo-advantage}
\end{equation}

Since the RM $r_\theta$ is only a proxy for true human preferences, the policy can exploit regions where it is inaccurate. As optimization progresses, proxy reward rises while true quality degrades, especially when the proxy RM is trained on limited preference data. This is known as \textit{reward hacking}~\citep{gao_scaling_2023}, and it worsens as the policy drifts from the distribution of the RM training data~\citep{ackermann_off-policy_2025}.

\subsection{Randomized smoothing}\label{sec:rs-background}

Reward modeling reduces to a binary classification problem on reward margins~\citep{tang_generalized_2024, nishimori2025symmetriclossesrobustpolicy}. Randomized smoothing (RS)~\citep{cohen_certified_2019} is a framework for certifying the robustness of a classifier's prediction against perturbations. Consider a classifier $f: \mathbb{R}^d \to [C]$ over $C$ classes. Given a noise scale $\sigma > 0$ for the Gaussian distribution $\mathcal{N}$, RS creates the \emph{smoothed classifier}
\begin{equation*}
	g(x):=\argmax_{c\in[C]} \Pr\big(f(x+\xi)=c\big),\quad \text{where} \ \ \xi \sim \mathcal{N}(0, \sigma^2 I_d),
\end{equation*}
which returns the class most likely to be predicted under random perturbations of the input. \citet{cohen_certified_2019} showed that if one predicted class is sufficiently more likely than the others under these perturbations, then the prediction of the smoothed classifier is provably consistent throughout an $\ell_2$ ball around $x.$ The radius of this ball is called the \textit{certified radius.} That is, the certified radius is the maximum $\ell_2$ perturbation to the input under which the smoothed classifier's prediction is guaranteed to remain unchanged.

\section{Related work}\label{sec:related-work}

Reward hacking arises when the policy distribution drifts outside the region where the proxy RM is reliable. Prior work addresses this through ensemble-based uncertainty, modified RM training, or adversarial robustness.

\paragraph{RM ensembles.}
Ensemble-based methods mitigate reward hacking by optimizing a policy conservatively with respect to multiple RMs. \citet{coste_reward_2023} train several proxy RMs and study conservative objectives such as worst-case and uncertainty-weighted optimization (UWO), showing that ensembles reduce reward hacking in the gold-RM setup of~\citet{gao_scaling_2023}. \citet{eisenstein_helping_2024} observe that ensembles still exploit spurious features when all members agree on them. \citet{zhai_uncertainty-penalized_2026} propose LoRA~\citep{hu_lora_2021}-based RM ensembles to reduce the cost of training multiple full RMs. These approaches require multiple RM training and inference runs.

\paragraph{Distribution shift and modified RM training.}
Reward hacking can also be viewed as a distribution shift problem: as the policy moves away from the SFT distribution, the RM's training data becomes unrepresentative. \citet{ackermann_off-policy_2025} propose off-policy corrected reward modeling, which uses importance weighting to make the RM consistent under the evolving policy distribution. \citet{dai_mitigating_2024} propose BSPO, which regularizes the RM training objective to improve robustness. \citet{xu_uncertainty-aware_2024} train distributional RMs with explicit uncertainty outputs. InfoRM~\citep{miao_inform_2024} trains the RM with an information bottleneck objective. These methods modify the RM training pipeline, whereas \method{} operates purely at policy optimization time using any pre-trained RM.

\paragraph{Adversarial policy optimization.}
The closest work to ours is AdvPO~\citep{zhang_mitigating_2024}, which derives single-RM uncertainty estimates from last-layer embeddings and formulates a distributionally robust policy optimization problem. Two differences distinguish our approach. First, AdvPO requires constructing a matrix $M_D$ from the RM training dataset $D$ at policy optimization time; in many practical pipelines, this data is unavailable when using third-party RMs~\citep{malik_rewardbench_2025, chijiwa_portable_2025}. \method{} requires only the current head parameters and on-policy completions. Second, AdvPO constructs a confidence region for the reward function using that matrix and then optimizes the policy against a pessimistic reward within that region. In contrast, \method{} models uncertainty as perturbations of the RM parameters and asks whether those perturbations can flip the sign of the group-relative advantage for each completion. This aligns our adversarial formulation directly with GRPO's update geometry, as we will discuss in \S\ref{sec:global-robust}.

Among the methods discussed above, all but AdvPO are ante-hoc, which operate by modifying the reward modeling pipeline or re-training RMs during policy optimization. AdvPO is partially post-hoc, in that it acts at the policy optimization stage but still requires prior knowledge about the reward modeling stage (i.e., the RM training dataset). In the next section, we introduce our proposed method as a purely post-hoc algorithm. Appendix~\ref{app:ocrm} also shows that our proposed method can be combined with ante-hoc methods.

\section{Policy optimization with certified advantage sign robustness}\label{sec:method}

We argue that when the RM $r_\theta$ assigns incorrect advantage signs, the policy is not properly updated and thus reward hacking can occur. Guided by the notion of the certified radius~\citep{cohen_certified_2019}, we formalize this by introducing the \textit{certified sign-preservation radius}, the smallest perturbation of the RM parameters that flips a completion's advantage sign (\S\ref{sec:main-results}). This leads to our proposed method of re-weighting the policy gradient updates, using the radius as a reliability measure (\S\ref{sec:setup}). Finally, we compare our method of using per-completion perturbations to the standard global perturbation setup, which we show suffers uniform pessimism unlike our method (\S\ref{sec:global-robust}). We focus on the advantage estimate used in Dr.GRPO \eqref{eq:drgrpo-advantage} throughout the rest of the paper, but we note that other PG methods can be used instead, as shown in Appendix~\ref{app:extension-pg}.

\subsection{Certified sign-preservation radius}\label{sec:main-results}

The PG~\eqref{eq:policy-gradient} is a weighted sum of per-completion directions:
\begin{equation}
	\nabla_\phi J(\phi, \theta)
	= \mathbb{E}_{x\sim P, \{y^{(1)},\dots,y^{(K)}\} \sim \pi_\phi}
	\!\left[
		\sum_{j=1}^{K} A_j(\theta) \, \nabla_\phi \log \pi_\phi(y^{(j)} | x)
		\right].
	\label{eq:grpo-pg}
\end{equation}
Each term pushes the policy to increase ($A_j(\theta) > 0$) or decrease ($A_j(\theta) < 0$) the probability of the $j$-th completion. Consider a single gradient step $\phi'=\phi + \eta \nabla_\phi J(\phi, \theta)$ with step size $\eta > 0$. Then, a first-order expansion gives
\begin{equation*}
	\log \pi_{\phi'}(y^{(j)} | x) - \log \pi_\phi(y^{(j)} | x)
	\approx \eta A_j(\theta) \| \nabla_\phi \log \pi_\phi(y^{(j)} | x) \|_2^2.
\end{equation*}
Since $\eta > 0$ and $\| \nabla_\phi \log \pi_\phi \|_2^2 \geq 0$, the sign of the advantage term $A_j(\theta)$ alone determines the update direction, when considering the $j$-th term in isolation. Crucially, the reliability of this sign prediction varies per completion: since the RM is only a proxy, some predictions in a group may not agree with the true reward, while others do agree (Figure~\ref{fig:overview}a).

To quantify the sensitivity for a given completion, we follow the idea of RS,
which certifies a decision under perturbations of a chosen variable. Given the discussion above, the decision that governs the policy update is the advantage sign, and the quantity we are uncertain about is the proxy RM itself. We therefore perturb the proxy RM parameters, and model its epistemic uncertainty through the uncertainty set: $\mathcal{U}_\tau^\theta:=\{\theta':\|\theta'-\theta\|_2\le\tau\}$~\citep{ben-tal_robust_2009}.

\begin{definition}[Certified sign-preservation radius]\label{def:cert-radius}
	Under an uncertainty set $\mathcal{U}_\tau^\theta$, the \emph{certified sign-preservation radius} for the $j$-th completion is
	\begin{equation}
		\Delta_j := \sup\bigl\{\tau \geq 0 :
		\mathrm{sign}(A_j(\theta')) = \mathrm{sign}(A_j(\theta))\
		\forall \theta' \in \mathcal{U}_\tau^\theta\bigr\}.
		\label{eq:cert-radius-def}
	\end{equation}
\end{definition}

This definition follows the certified-radius framework of \citet{cohen_certified_2019}, where the advantage sign is viewed as a binary classification decision in the RM parameter space (see Appendix~\ref{app:randomized-smoothing} for the formal connection). As such, \eqref{eq:cert-radius-def} has a natural interpretation as a sensitivity measure. A completion with large $\Delta_j$ has a \textit{robust advantage sign}: the RM would need to change substantially before the sign flips, so the corresponding gradient term in~\eqref{eq:grpo-pg} can be trusted. On the other hand, for a completion with small $\Delta_j$, a small perturbation to the RM could reverse whether the policy reinforces or suppresses it. The certified sign-preservation radius thus provides a per-sample confidence measure: for each completion individually, it tells us how much the update direction depends on the exact choice of the proxy RM parameters $\theta$.

For a general differentiable RM $r_\theta$ with full-parameter uncertainty $\mathcal{U}_\tau^\theta$, a first-order Taylor expansion of the advantage around $\theta$ gives $A_j(\theta') \approx A_j(\theta) + \nabla_\theta A_j(\theta)^\mathsf{T}(\theta' - \theta)$. The certified sign-preservation radius is then
$\Delta_j^\mathrm{full}(\theta):=\frac{|A_j(\theta)|}
	{\|\nabla_\theta A_j(\theta)\|_2}.$
Unfortunately, computing $\|\nabla_\theta A_j(\theta)\|_2$ requires the per-completion gradient norm of the RM, which our preliminary experiments found to be computationally infeasible for larger RMs.
We thus propose to use a simpler perturbation model, for which we can efficiently calculate the certified sign-preservation radius, as we will show below.

\subsection{\methodlong{} (\method{})}\label{sec:setup}

We consider the RM architecture $r_\theta(x,y) = w^\mathsf{T} h_\psi(x,y) + b$, consisting of a feature extractor $h_\psi(x,y)$ with parameters $\psi$ and a linear head with parameters $(w,b)$, as commonly used in practice~\citep{stiennon_learning_2020, ouyang_training_2022}.
Instead of perturbing all parameters $(\psi,w,b)$, we only perturb the linear head, and thus use the uncertainty set\footnote{We discuss other possible choices of perturbation models and their relationship in Appendix~\ref{sec:other-models}.}
\begin{equation}
	\mathcal{U}_\epsilon^w := \{w' \in \mathbb{R}^d : \|w' - w\|_2 \le \epsilon\}.
	\label{eq:uncertainty-head}
\end{equation}

In what follows, $A_j(w)$ denotes $A_j(\theta),$ stressing perturbations for the head parameters $w.$

Under the linear head model, the certified sign-preservation radius~\eqref{eq:cert-radius-def} admits a closed form that is easy to compute during policy optimization without additional information:
\begin{theorem}[Certified radius]\label{thm:robustness}
	Under the linear head model and the uncertainty set $\mathcal{U}_\epsilon^w$ \eqref{eq:uncertainty-head}, the certified sign-preservation radius is
	\begin{equation}
		\Delta_j = \frac{|A_j(w)|}{\|h_\psi(x,y^{(j)}) - \bar{h}\|_2}, \quad \text{where} \ \ \bar{h}:=\frac{1}{K}\sum_{k=1}^Kh_\psi(x,y^{(k)}).
		\label{eq:certified-radius-param}
	\end{equation}
\end{theorem}

The proof follows from the linearity of the reward head and is given in Appendix~\ref{app:proofs}.

Based on the certified radius above, we propose \textit{\method{}}, a conservative variant of the standard PG \eqref{eq:grpo-pg}. For each completion, let the adversary choose $w' \in \mathcal{U}_\epsilon^w$ to maximally corrupt that completion's gradient contribution. Concretely, the adversary minimizes the advantage for completions the policy with the nominal RM would reinforce ($A_j(w) > 0$) and maximizes it for completions the policy would suppress ($A_j(w) < 0$):
\begin{equation}
	\tilde{A}_j :=
	\begin{cases}
		\inf_{w' \in \mathcal{U}_\epsilon^w} A_j(w') & \text{if } A_j(w) > 0, \\[4pt]
		\sup_{w' \in \mathcal{U}_\epsilon^w} A_j(w') & \text{if } A_j(w) < 0.
	\end{cases}
	\label{eq:per-sample-adversary}
\end{equation}
This formulation asks: Given that the RM could be wrong within $\mathcal{U}_\epsilon^w$, what is the most conservative advantage we can still attribute \emph{to each completion}?
The group mean $\bar{h}$ enters every $A_j$, but each adversary chooses its perturbation to flip the advantage sign of one completion only, so the adversaries of different completions are not coupled. Appendix~\ref{app:per-sample-adversary} compares this with a single adversary shared across the group.

\begin{theorem}[Worst-case advantage under per-completion adversary]\label{lem:worst-case-advantage}
	Under the linear head model and the uncertainty set $\mathcal{U}_\epsilon^w$~\eqref{eq:uncertainty-head}, the worst-case advantage~\eqref{eq:per-sample-adversary} for the $j$-th completion is
	\begin{equation}
		\tilde{A}_j = \rho_j^* \cdot A_j(w), \quad \text{where} \quad
		\rho_j^* := 1 - \frac{\epsilon}{\Delta_j}
		= 1 - \frac{\epsilon\,\|h_\psi(x,y^{(j)}) - \bar{h}\|_2}{|A_j(w)|}.
		\label{eq:optimal-weight}
	\end{equation}
\end{theorem}
The proof is given in Appendix~\ref{app:proofs}. \method{} thus replaces each advantage with its worst-case value:
\begin{equation}
	\widehat{\nabla}_\phi J(\phi, \theta)
	= \mathbb{E}_{x}\left[
		\sum_{j=1}^K \rho_j^* \cdot A_j(w)
		\cdot \nabla_\phi \log\pi_\phi(y^{(j)}|x)
		\right].
	\label{eq:conservative-pg}
\end{equation}
This is the standard GRPO gradient with each term re-weighted using its certified sign-preservation radius. Setting $\epsilon = 0$ recovers standard GRPO. \eqref{eq:conservative-pg} can also be derived as the PG of a certain robust optimization problem but the discussion is deferred to Appendix~\ref{app:per-sample-adversary}.

\paragraph{Practical implementation.}\label{sec:algorithm}

The coefficient $\rho_j^*$ \eqref{eq:optimal-weight} depends only on the advantage $A_j(w)$ and the feature deviation $\|h_\psi(x,y^{(j)}) - \bar{h}\|_2$.
Rather than fixing $\epsilon$ globally, we set it adaptively using the empirical distribution of the ratio $|A_j(w)|/\|h_\psi(x,y^{(j)}) - \bar{h}\|_2$ within each batch.
Specifically, we set $\epsilon$ to the $q_t$-th quantile of $\{|A_j(w)| / \|h_\psi(x,y^{(j)}) - \bar{h}\|_2\}_{j=1}^{B\cdot K}$ over all $B \cdot K$ completions in the current batch, where $B$ is the number of prompts and $K$ is the group size.
We neither clip $\rho_j^*$ at zero nor re-normalize it within a group.
A fraction $q_t$ of the completions in a batch thus have $\Delta_j < \epsilon$ and $\rho_j^* < 0$.
For such a completion, the uncertainty set contains an RM head under which its advantage has the opposite sign, and the update treats it as lying on the other side of the group mean.
As shown in Appendix~\ref{app:ablation}, the parameter $q_t \in [0,1]$ provides a natural interpretation of controlling the proxy RM accuracy during policy optimization.
The full pseudocode is provided in Algorithm~\ref{alg:rapo} (Appendix~\ref{app:algorithm}).

\subsection{Comparison with the global robust objective}\label{sec:global-robust}

Finally, we compare our per-completion adversary~\eqref{eq:per-sample-adversary} with the standard approach of taking a single worst-case reward for all completions over an uncertainty set~\citep{ben-tal_robust_2009, zhang_mitigating_2024}, which we show that cannot distinguish reliable from unreliable completions. Consider the \textit{global} robust objective, with the shared perturbation $w'$:
\begin{equation}
	J_\mathrm{global}(\phi,\theta) = J_\mathrm{global}(\phi,\{\psi,w,b\}):=\inf_{w'\in\mathcal{U}_\epsilon^w}J(\phi,\{\psi, w', b\}).
	\label{eq:global-robust}
\end{equation}

The following lemma shows that this approach falls short (the proof is given in Appendix~\ref{app:proofs}):

\begin{lemma}[Policy gradient of the global robust objective]
	\label{lem:global-pg}
	Under the linear reward head, the PG of the worst-case objective \eqref{eq:global-robust} is
	\begin{equation}
		\nabla_\phi J_{\mathrm{global}}(\phi,\theta)
		= \mathbb{E}_x \left[ \sum_{j=1}^{K} A_j(w)  \nabla_\phi \log \pi_\phi(y^{(j)} | x) \right]
		- \epsilon  \nabla_\phi \bigl\| \mathbb{E}_{x,y}[h_\psi(x, y)] \bigr\|_2.
		\label{eq:global-pg}
	\end{equation}
\end{lemma}

Here, the correction term $\nabla_\phi \| \mathbb{E}_{x,y}[h_\psi(x,y)] \|_2$ is a single direction shared across all completions $y\sim\pi_\phi.$ It does not depend on \emph{which} completions receive reliable reward signal and which do not: a completion whose ranking is highly certain receives the same penalty as one whose ranking could be flipped by a small perturbation of the RM parameters.

This uniform pessimism is also observed beyond the linear-head case. \citet{zhang_mitigating_2024} empirically observed that their uncertainty set over the whole RM parameters $\theta$ led to overly pessimistic reward corrections. \citet{eisenstein_helping_2024} similarly found that ensemble-based penalties fail when all members agree on the same spurious features, suggesting that completion-level discrimination is necessary.

In contrast, our per-completion formulation~\eqref{eq:per-sample-adversary} and re-weighting~\eqref{eq:conservative-pg} provide this discrimination: completions with robust advantage signs receive the gradient essentially unchanged, while completions with unreliable signs receive small or negative weights. \method{} therefore selectively suppresses or reverses the unreliable gradient directions while preserving the reliable ones.

\section{Experiments}\label{sec:experiments}

To evaluate our proposed method in LLM alignment tasks, we adopt the gold-model setup of~\citet{gao_scaling_2023}, which allows us to investigate reward hacking by simulating the true reward $r^*$ with a larger ``gold'' RM.
The pipeline proceeds in four stages:
\begin{enumerate}
	\item \textbf{SFT.}  A base LLM is fine-tuned on reference responses to obtain the SFT policy $\pi_{\mathrm{SFT}}$.
	\item \textbf{Preference data generation.}  For each prompt $x$, two candidate responses are sampled from $\pi_{\mathrm{SFT}}$ and labeled by the gold RM to form pairwise preference data.
	\item \textbf{Proxy RM training.}  A proxy reward model is trained on the generated preference data using the Bradley--Terry objective \eqref{eq:bt-loss}. We add a randomly initialized scalar head on top of the same base LLM backbone as the SFT model $\pi_\mathrm{SFT}$. During RM training, although we derived~\eqref{eq:cert-radius-def} under the frozen-backbone assumption, we train the full RM parameters to test broader applicability of our proposed method.
	\item \textbf{Policy optimization and evaluation.}  The SFT policy is optimized against the proxy RM using one of the methods described below. The resulting policy is evaluated by the gold RM that was not used during policy optimization.
\end{enumerate}

We conduct experiments on the foundational summarization from human feedback task (TL;DR)~\citep{stiennon_learning_2020} and on the AlpacaFarm benchmark~\citep{dubois_alpacafarm_2023}. We use Pythia~\citep{biderman_pythia_2023} and Qwen2.5~\citep{qwen_qwen25_2025} series as base models. For Pythia 1B and Qwen2.5 1.5B policies, the gold RM is Skywork-Reward-Llama-3.1-8B-v0.2~\citep{liu_skywork-reward_2024}, one of the top-performing discriminative RMs on RewardBench~\citep{malik_rewardbench_2025}. For larger policy models (Pythia 2.8B and Qwen2.5 3B), we use OpenAI's GPT-4.1 Nano as the gold RM. More experimental details are given in Appendix~\ref{app:algorithm}.

\subsection{Comparison with baselines}\label{sec:main-comparison}

\begin{table}[t]
	\caption{Win rate (\%) vs.\ human-generated reference response after training. Entries with $\pm$ report mean $\pm$ standard deviation over three random seeds. Entries without $\pm$ are from a single seed. We abbreviate Pythia as P and Qwen2.5 as Q. Best result per column is shown in \textbf{bold}.}
	\label{tab:main-results}
	\centering
	\begin{tabular}{lcccccc}
		\toprule
		Dataset          & \multicolumn{4}{c}{TL;DR Summarization} & \multicolumn{2}{c}{AlpacaFarm}                                                                 \\
		\cmidrule(lr){2-5} \cmidrule(lr){6-7}
		Model            & P 1B                                    & Q 1.5B                         & P 2.8B        & Q 3B          & Q 1.5B        & Q 3B          \\
		\midrule
SFT              & $15.7 \pm 0.2$                                    & $27.3\pm0.3$                           & $25.2\pm0.4$          & $35.6\pm0.4$          & $37.8\pm0.7$          & $47.9\pm0.6$          \\
		Dr.GRPO          & $21.0 \pm 5.4$                          & $66.9 \pm 3.1$                 & $63.6 \pm 6.5$ & $89.9 \pm 2.3$ & $26.9 \pm 5.1$ & $42.3 \pm 4.0$ \\
		UWO              & $6.1\pm3.6$                             & $59.6 \pm 5.3$                 & 38.1          & 81.8          & 38.3          & 46.9          \\
		BSPO             & $40.4\pm1.2$                            & $\mathbf{78.3\pm0.7}$          & 50.0          & 63.1          & 31.3          & 43.8          \\
		AdvPO            & $47.0 \pm 2.6$                          & $58.1 \pm 1.0$                 & 69.1          & 90.0          & $36.0 \pm 7.0$ & $48.6 \pm 0.9$ \\
		\midrule
		\method{} (ours) & $\mathbf{60.0 \pm 2.0}$                 & $66.0 \pm 1.9$                 & $\mathbf{72.0 \pm 4.2}$ & $\mathbf{91.2 \pm 0.8}$ & $\mathbf{48.6 \pm 2.1}$ & $\mathbf{50.8 \pm 3.4}$ \\
		\bottomrule
	\end{tabular}
\end{table}

\begin{table}[t]
    \caption{Average wall-clock time per policy optimization step (seconds) on 8 GPUs over three runs on TL;DR dataset, as in Table \ref{tab:main-results}. \method{} roughly matches Dr.GRPO.}
	\label{tab:runtime}
    \centering
	\begin{tabular}{lcc}
		\toprule
		Method    & Pythia 1B & Qwen2.5 1.5B \\
		\midrule
		Dr.GRPO   & 7.54      & 10.67        \\
		UWO       & 7.48      & 10.76        \\
		BSPO      & 9.92      & 13.37        \\
		AdvPO     & 10.93     & 16.32        \\
		\midrule
		\method{} & 7.47      & 10.78        \\
		\bottomrule
	\end{tabular}
\end{table}

We compare our proposed method with the SFT policy $\pi_\mathrm{SFT},$ Dr.GRPO~\citep{liu_understanding_2025}, Uncertainty Weighted Optimization (UWO)~\citep{coste_reward_2023}, Behavior-Supported Policy Optimization (BSPO)~\citep{dai_mitigating_2024}, and AdvPO~\citep{zhang_mitigating_2024}. For UWO, we use an ensemble of three RMs and smaller individual RMs to control total model capacity, following the setup of \citet{zhang_mitigating_2024}. For every method, including \method{}, we select the hyperparameters by the gold-RM win rate (Appendix~\ref{sec:baselines}). We also compare against Direct Preference Optimization (DPO)~\citep{rafailov_direct_2023} in Appendix~\ref{app:dpo}.

Table~\ref{tab:main-results} presents the gold-RM win rate for each method across both benchmarks and model scales. On TL;DR summarization, \method{} achieves the highest win rate in most settings. 
We observe larger gains when the proxy RM is relatively small, which we believe is due to smaller proxy RMs becoming inaccurate faster (Appendix~\ref{app:rm-accuracy}) and thus needing robustness more. As the policy and proxy RM scale up, the margin narrows but \method{} remains competitive or superior. Results for a larger 6.9B policy are reported in Appendix~\ref{app:larger-policy}. We also note that the performance of UWO on the TL;DR task with Pythia 1B does not improve even with larger proxy RMs (Table~\ref{tab:uwo-rm-size} in Appendix~\ref{app:additional-baselines}). 

\paragraph{Runtime.}
Table~\ref{tab:runtime} compares the per-step wall-clock time across methods.
\method{} adds negligible overhead to Dr.GRPO: the only additional computation is evaluating the coefficients $\rho_j^*$~\eqref{eq:optimal-weight} from the RM's last-layer features, which are already extracted during reward scoring.
UWO matches Dr.GRPO in training time, but requires loading multiple ensemble members, increasing VRAM usage when using full-sized RMs.
BSPO and AdvPO incur additional training cost due to their modified optimization procedures.

\begin{figure}[t]
	\centering
	\begin{tabular}{cc}
		\includegraphics[width=0.48\textwidth]{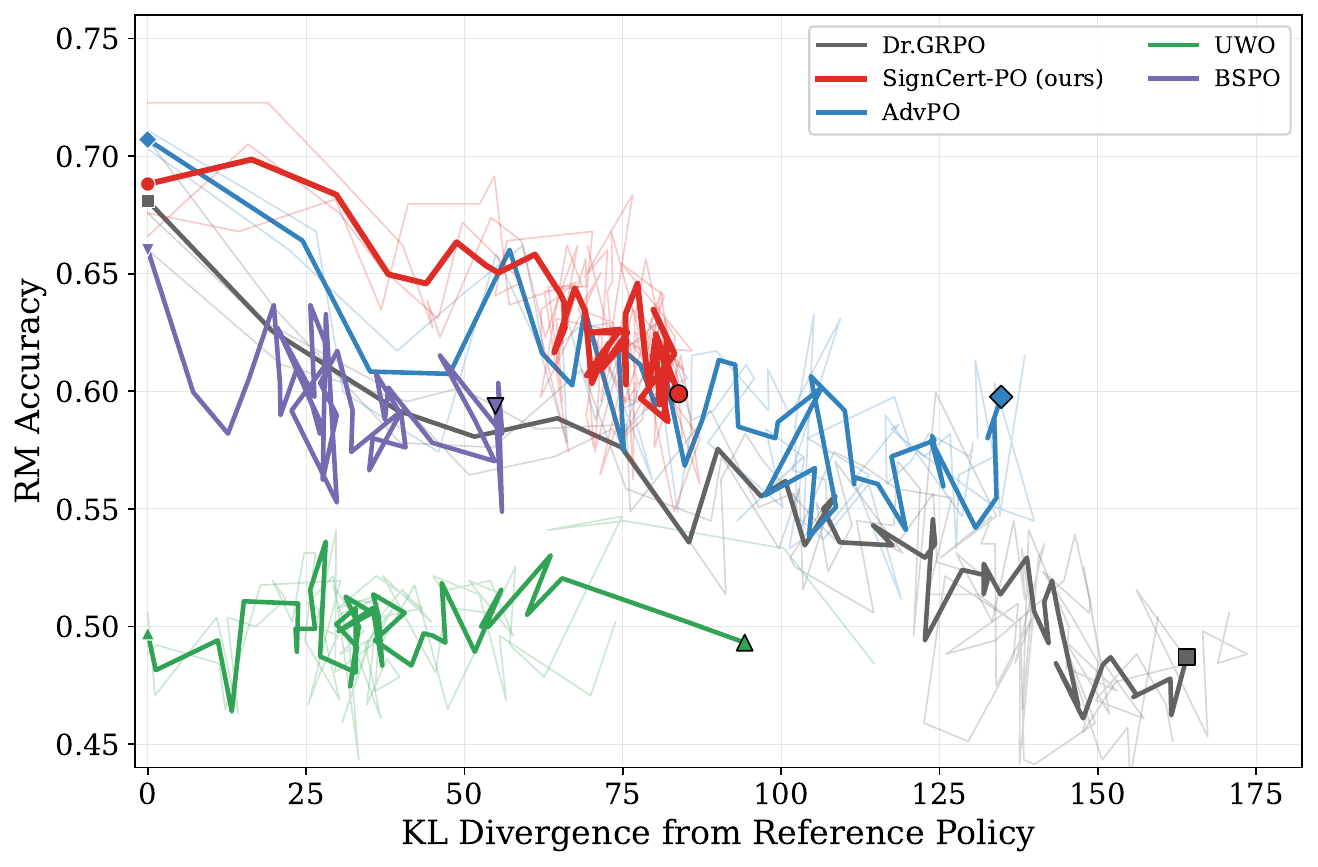} &
		\includegraphics[width=0.48\textwidth]{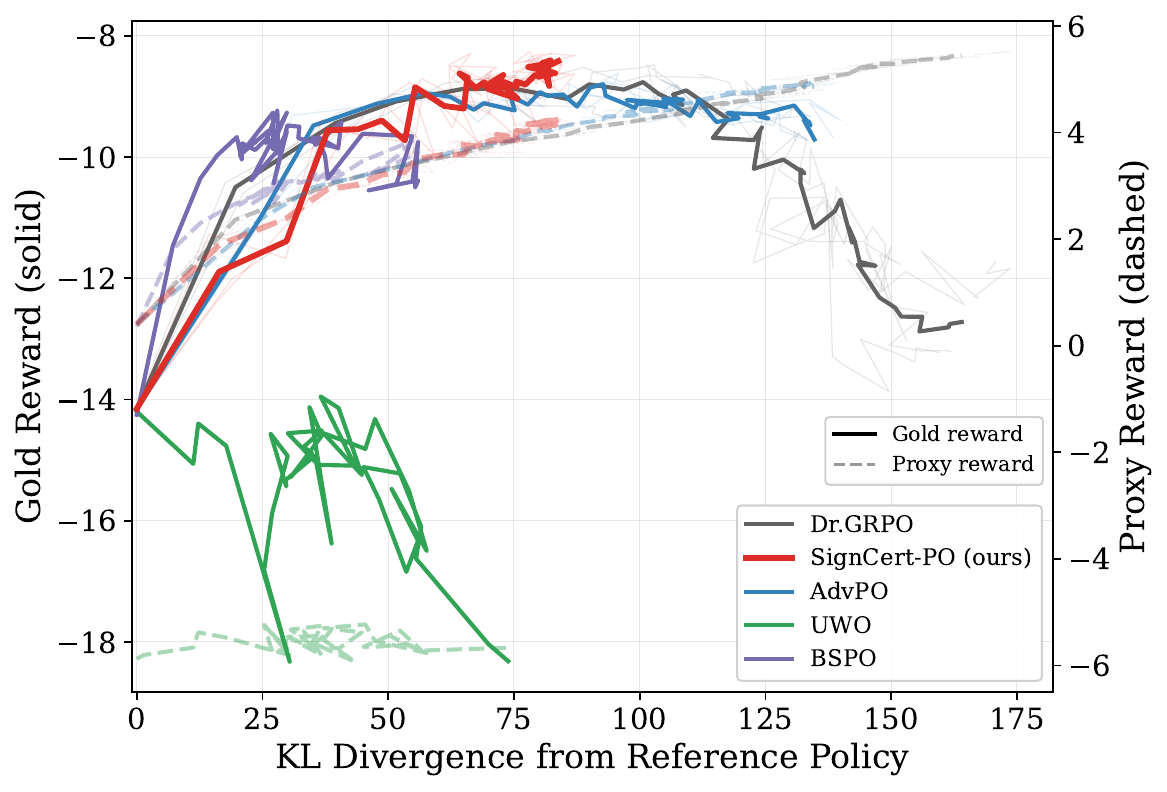}                                    \\
		(a) Pythia-1B: KL--RM accuracy                                                 & (b) Pythia-1B: KL--reward \\[1em]
	\end{tabular}
	\caption{\textbf{\method{} keeps the policy in regions where the proxy RM remains reliable, preventing reward hacking.} KL divergence trade-offs on TL;DR. Left: proxy RM accuracy vs.\ KL. \method{} maintains higher RM accuracy at every KL budget. Right: gold reward (solid) and proxy reward (dashed) vs.\ KL. Baselines exhibit reward hacking, whereas \method{} avoids this divergence. The reference policy is the SFT model $\pi_\mathrm{SFT}.$}
	\label{fig:kl-tradeoff}
\end{figure}

\begin{table}[t]
\caption{Win rate (\%) and gold reward after training across KL coefficients $\beta$ on TL;DR with Pythia 1B. Multi-seed cells report mean $\pm$ standard deviation over three seeds.}
	\label{tab:kl-sweep}
	\centering
	\small
	\begin{tabular}{lccc}
		\toprule
		Win rate / Gold reward & $\beta=0.001$ & $\beta=0.005$ & $\beta=0.01$ \\
		\midrule
Dr.GRPO   & $\mathbf{13.2}$ / $\mathbf{-15.0}$ & $11.1$ / $-15.0$ & $21.0 \pm 5.4$ / $-12.7 \pm 0.86$ \\
\method{} & $10.7$ / $-15.8$ & $\mathbf{40.4}$ / $\mathbf{-10.3}$ & $\mathbf{60.0 \pm 2.0}$ / $\mathbf{-8.42 \pm 0.12}$ \\
		\bottomrule
	\end{tabular}
\end{table}

\subsection{Additional experiments and analysis}\label{sec:analysis}

We further perform additional experiments to better understand our proposed method.

\paragraph{How \method{} mitigates reward hacking.}\label{sec:mechanism}

Figure~\ref{fig:kl-tradeoff}a
shows that as the policy drifts from $\pi_\mathrm{SFT}$ (i.e., increasing KL divergence), the proxy RM's accuracy 
on on-policy completions degrades.
Dr.GRPO's accuracy drops toward 50\% at high KL as a result, while \method{} follows a different trajectory. 
Our proposed method keeps the policy in regions where the proxy RM maintains accuracy around 60\%, even as optimization progresses. 
This is a consequence of the re-weighting mechanism: completions whose advantage signs are easy to flip receive reduced gradient weight, preventing the policy from being pulled toward regions where the proxy RM is unreliable.

This RM accuracy degradation can lead to reward hacking, as suggested by Figure~\ref{fig:kl-tradeoff}b. Dr.GRPO's proxy reward (dashed) continues to rise even as its gold reward (solid) collapses. \method{} achieves a comparable or even better gold reward at a fraction of the KL budget and avoids the collapse.

We also compared the advantage-sign error rate directly during optimization (Figure~\ref{fig:sign-error}). Under Dr.GRPO the sign-error rate grows as training proceeds. Under \method{}, it stays substantially lower throughout, consistent with the policy remaining in a region where the proxy RM tracks the gold RM.

\begin{wrapfigure}{r}{0.4\textwidth}
\vspace{-2em}
	\centering
	\includegraphics[width=0.4\textwidth]{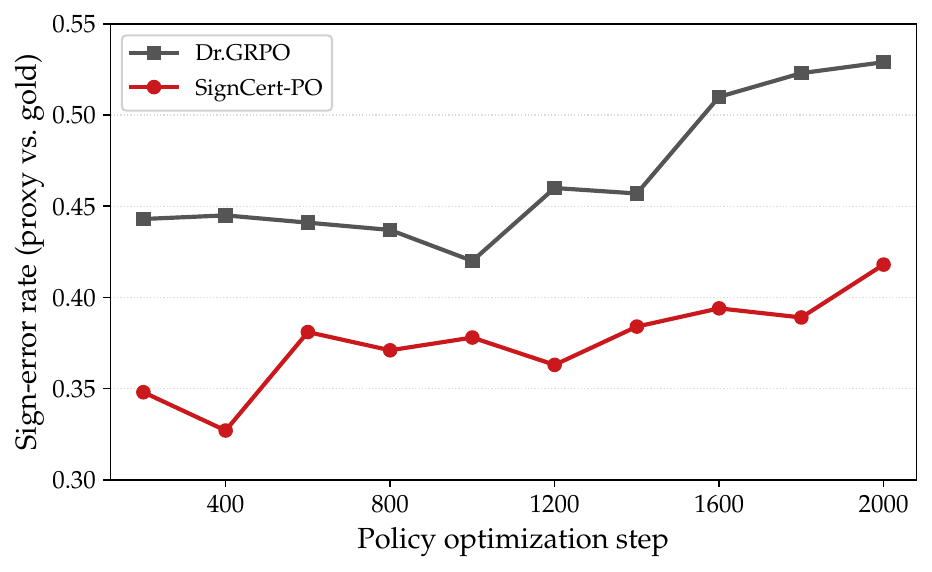}
\caption{\textbf{\method{} keeps the proxy RM's advantage signs closer to the gold RM during optimization.} Fraction of on-policy completions whose proxy advantage sign disagrees with the gold RM over training steps, on TL;DR with Pythia 1B.}
	\label{fig:sign-error}
\vspace{-2em}
\end{wrapfigure}
\paragraph{Comparison with early stopping and KL regularization.}
The gold reward curves in Figure~\ref{fig:kl-tradeoff}b also demonstrate the comparison with early stopping. Dr.GRPO with $\beta{=}0.01$ (the one also used for \method{}) reaches peak gold reward around KL$\approx$100, after which the reward degrades.
An ideal early-stopping mechanism that stops training at this point would recover its peak gold reward, which is still lower than the gold reward after training by \method{}. In contrast, \method{} converges to a favorable region without relying on external evaluation.

The KL penalty $\beta$ and \method{} both restrain the policy from drifting into unreliable regions, but they are not redundant. Table~\ref{tab:kl-sweep} sweeps $\beta$ for both methods on TL;DR with Pythia 1B. At a very small penalty ($\beta=0.001$), neither method prevents the gold reward from collapsing, showing that \method{} does not remove the need for KL regularization. For $\beta=0.005$ and $\beta=0.01$, \method{} achieves a higher win rate and higher gold reward than Dr.GRPO at the same penalty. The two mechanisms are therefore complementary: a nonzero KL penalty remains useful, and \method{} widens the region over which optimization does not collapse.

\begin{wraptable}{r}{0.45\textwidth}
\vspace{-2em}
	\caption{Win rate (\%) after 2{,}000 steps on TL;DR with Pythia 1B across group sizes $K$, at $q_t=0.4$.}
	\label{tab:group-size}
	\centering
	\small
	\begin{tabular}{lccc}
		\toprule
		Method    & $K=4$ & $K=8$ & $K=16$ \\
		\midrule
		Dr.GRPO   & $26.2$ & $21.0 \pm 5.4$ & $23.2$ \\
		\method{} & $\mathbf{54.7}$ & $\mathbf{60.0 \pm 2.0}$ & $\mathbf{57.4}$ \\
		\bottomrule
	\end{tabular}
\vspace{-2em}
\end{wraptable}

\paragraph{Robustness to the group size.}
\method{} computes $\Delta_j$ within groups of $K$ on-policy completions, so its behavior could depend on the group size. Table~\ref{tab:group-size} varies $K\in\{4,8,16\}$ on TL;DR with Pythia 1B at a fixed quantile parameter, holding the batch size constant. \method{} improves over Dr.GRPO at every group size, and its win rate is stable across $K$.

\begin{figure}[t]
	\begin{minipage}[t]{0.48\textwidth}
		\centering
		\includegraphics[width=\textwidth]{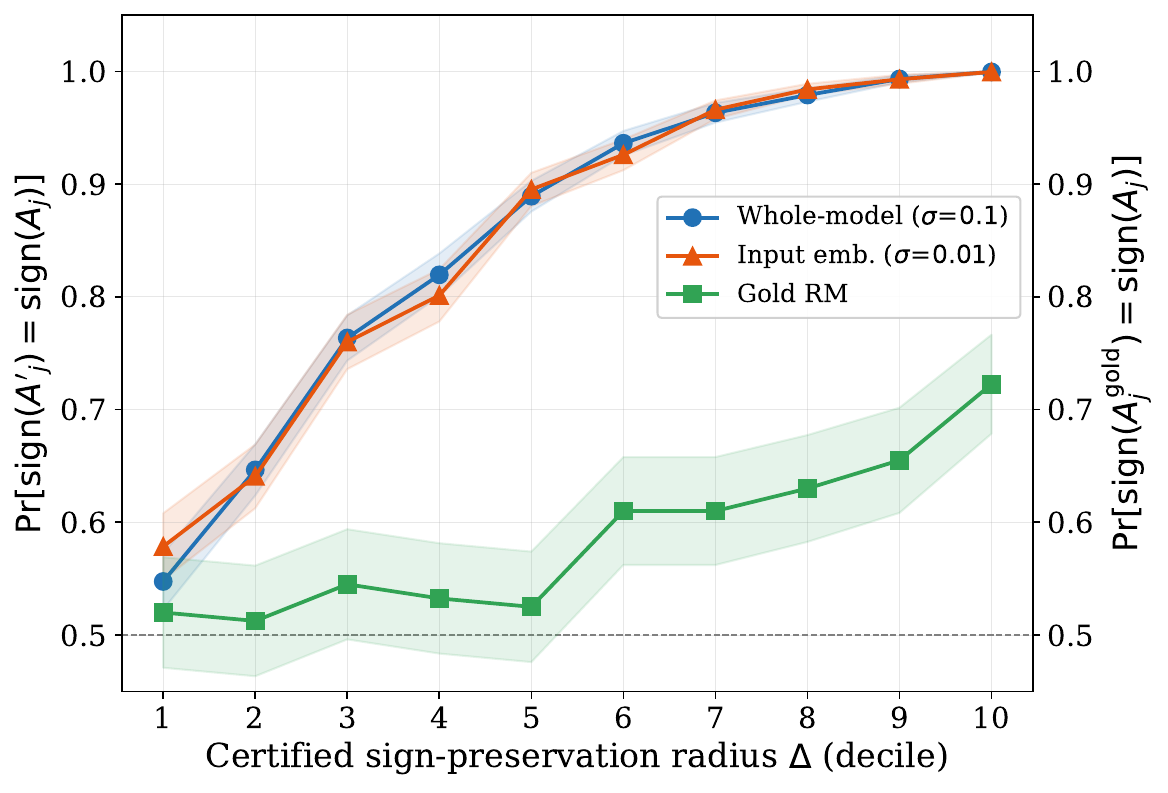}
		\caption{\textbf{$\Delta_j$ \eqref{eq:certified-radius-param} is predictive of sign robustness beyond the linear head assumption,} on the TL;DR task for Pythia 1B. Left axis: agreement with other perturbation models, where $A'_j$ is the advantage recomputed under whole-RM or input embedding perturbation. Right axis: agreement with the gold RM. See Appendix \ref{app:other-experimental-details} for details.}
		\label{fig:radius_validation}
	\end{minipage}
	\hfill
	\begin{minipage}[t]{0.48\textwidth}
    	\centering
	\includegraphics[width=\textwidth]{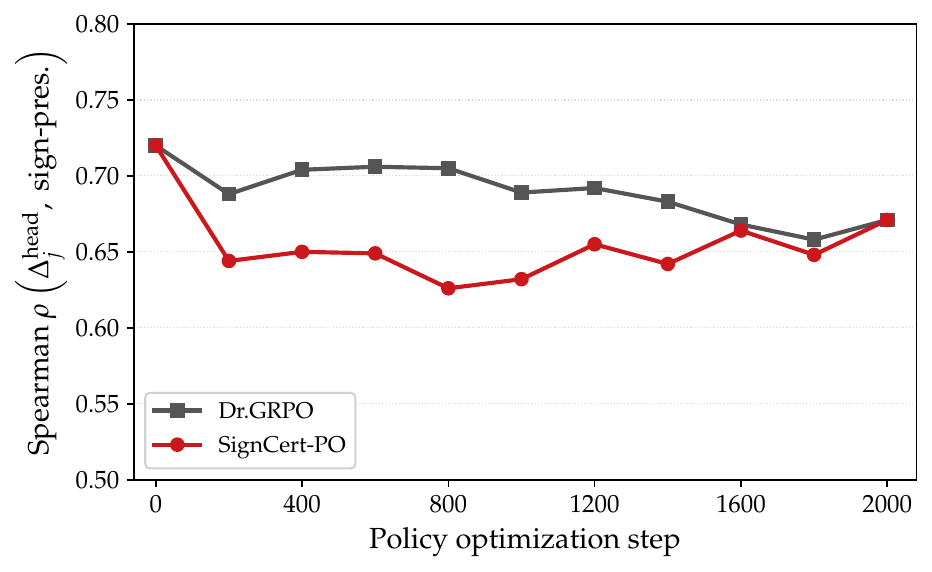}
	\caption{\textbf{The head radius keeps predicting full-parameter sign robustness throughout optimization.} Rank correlation between $\Delta_j$ and the full-parameter sign-preservation rate at each checkpoint, on TL;DR with Pythia 1B, for Dr.GRPO and \method{}.}
	\label{fig:radius-trajectory}
	\end{minipage}
\end{figure}

\begin{figure}[t]
		\centering
		\includegraphics[width=0.55\textwidth]{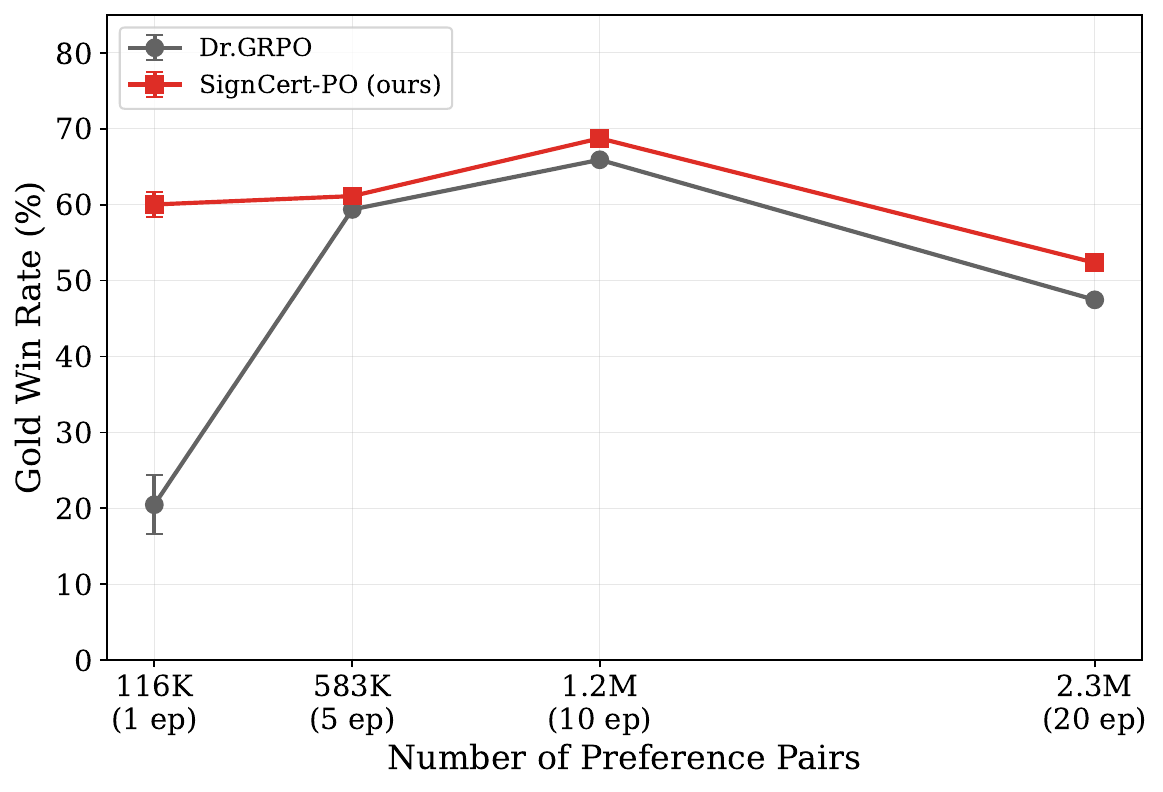}
		\caption{\textbf{\method{} provides the largest gains when preference data is limited, with the gap narrowing as more data becomes available.} Gold win rate vs.\ number of preference data epochs on TL;DR for the Pythia 1B proxy RM. We also observe overfitting of the proxy RM for 2.3M pairs.}
		\label{fig:rm_epoch_winrate}
\end{figure}

\paragraph{Empirical validation of the simplifying perturbation model.}
In \S\ref{sec:method}, we assumed linear head perturbation when deriving the certified sign-preservation radius $\Delta_j$ \eqref{eq:certified-radius-param}, in order to simplify computation during policy optimization. 
We thus test whether $\Delta_j$ predicts robustness under the other perturbation classes discussed in \S\ref{sec:main-results} and show the results in Figure~\ref{fig:radius_validation}.
We show (i) whole-model perturbation, where all RM parameters are perturbed, (ii) input embedding perturbation, where noise is added before the transformer backbone, and (iii) cross-model agreement, measuring whether the proxy RM's advantage signs match those of the gold RM. In all three settings, completions with larger $\Delta_j$ consistently demonstrate higher sign-preservation or agreement rates. Furthermore, the Spearman rank correlation coefficient between $\Delta_j$ and sign preservation rate is $0.72$ for whole-model perturbation and $0.67$ for input-embedding perturbation. These results suggest that $\Delta_j$ captures sensitivity structure beyond our theoretical assumption.

The validation above is computed at the reference policy $\pi_\mathrm{SFT}$. Because \method{} uses $\Delta_j$ throughout optimization, we track the same relationship for the saved checkpoints during optimization. Figure~\ref{fig:radius-trajectory} reports the rank correlation between $\Delta_j$ and the sign-preservation rate under full-parameter perturbations. The correlation stays stable across training for both Dr.GRPO and \method{}, suggesting that the head radius keeps identifying sign-robust completions throughout optimization.

\paragraph{Sensitivity to proxy RM quality.}
\label{sec:rm_quality}

The proxy RM's reliability depends on the amount of preference data used for its training. Figure~\ref{fig:rm_epoch_winrate} varies the number of generated preference pairs (from 116K at 1 epoch to 2.3M at 20 epochs) and compares the win rate of Dr.GRPO and \method{} on Pythia 1B. \method{} provides the largest gains when the preference data is limited. With only 116K preference pairs (i.e., one pair per prompt in the dataset), \method{} achieves significantly higher win rate compared to Dr.GRPO. This is the most practically relevant regime, since collecting large preference datasets is expensive. 
While the benefit of our method decreases for larger datasets, results show that it does not hurt performance when the RM is already reliable. 
Finally, for more than 1M pairs we observe overfitting of the proxy RM and the results for both methods degrade.

\section{Conclusion}

We presented \method{}, a lightweight method for RLHF that mitigates reward hacking by down-weighting policy gradient contributions of completions whose advantage signs are fragile under RM parameter perturbations. Our approach derives a certified sign-preservation radius that quantifies, for each completion, the smallest perturbation to the RM that can flip the advantage sign. This per-completion formulation yields a simple re-weighting mechanism of the standard PG update, requiring only the current RM head parameters and on-policy completions.  Experiments on TL;DR summarization and AlpacaFarm across Pythia and Qwen2.5 model families show that \method{} achieves the highest gold-RM win rate in most settings while adding negligible computational overhead.

Our method makes several simplifying assumptions (\S\ref{sec:setup}). It also targets only errors in the advantage sign. Errors in reward magnitude can matter as well, since a large error on one completion shifts the group mean and can flip the advantage signs of the others. Further, our per-sample adversary formulation is more conservative than a shared-adversary variant. Analyzing the tighter coupled formulation is a future direction. Finally, extending the certified sign-preservation radius to general RMs without incurring the cost of full-parameter gradient norm computation remains an open problem.

\section*{Broader Impact Statement}

This work proposes a new algorithm that mitigates reward hacking in RLHF, contributing to the development of more reliable and aligned language models. Our method does not involve the collection of new human preference data or human subjects. We use only publicly available datasets and models in our experiments. We do not foresee specific ethical risks arising from this work beyond those inherent to the RLHF framework.

\section*{Acknowledgments}

This research was conducted using the Supermicro ARS-111GL-DNHR-LCC and FUJITSU Server PRIMERGY CX2550 M7 (Miyabi) at Joint Center for Advanced High Performance Computing (JCAHPC).
MS was supported by JST ASPIRE Grant Number JPMJAP25B1.
TI was supported by KAKENHI Grant Number 22K17946.

\bibliographystyle{tmlr}
\bibliography{references}

\appendix

\section{Additional background}\label{app:additional-background}

We present a more comprehensive discussion of the background here.

\subsection{Policy gradient algorithms}\label{app:pg-algorithms}

\paragraph{Proximal Policy Optimization (PPO).}
PPO~\citep{schulman_proximal_2017} is a canonical PG-based method for RLHF that uses a clipped surrogate objective and jointly trains a value network for advantage estimation. Given responses sampled from an old policy $\pi_{\phi_\mathrm{old}}$, PPO maximizes the clipped surrogate objective:
\begin{equation}
	L^{\mathrm{PPO}}(\phi) = \mathbb{E}_{x \sim P, y \sim \pi_{\phi_\mathrm{old}}} \left[ \frac{1}{T} \sum_{t=1}^{T} \min\Big( \nu_t(\phi) \hat{A}_t, \mathrm{clip}\big(\nu_t(\phi), 1 - \varepsilon_\mathrm{clip}, 1 + \varepsilon_\mathrm{clip}\big) \hat{A}_t \Big) \right],
	\label{eq:ppo-objective}
\end{equation}
where $T = |y|$ is the response length, $\varepsilon_\mathrm{clip} > 0$ is the clipping hyperparameter, and
\begin{equation*}
	\nu_t(\phi) := \frac{\pi_\phi(y_t \mid x, y_{<t})}{\pi_{\phi_\mathrm{old}}(y_t \mid x, y_{<t})}
\end{equation*}
is the density ratio. The token-level advantage $\hat{A}_t$ is estimated using a learned value network $V_\omega(x, y_{<t})$ with parameters $\omega$, trained jointly with the policy:
\begin{equation}
	\hat{A}_t = r_t + \gamma V_\omega(x, y_{\le t}) - V_\omega(x, y_{<t}),
	\label{eq:ppo-advantage}
\end{equation}
where $\gamma \in [0,1]$ is the discount factor and $r_t$ is the per-token reward. In RLHF, $r_t = -\beta \log\bigl(\pi_\phi(y_t \mid x, y_{<t}) / \pi_\mathrm{ref}(y_t \mid x, y_{<t})\bigr)$ for all $t$, with the RM score $r_\theta(x, y)$ added at the final token $t = T$~\citep{ziegler_fine-tuning_2020}.

\paragraph{Group Relative Policy Optimization (GRPO).}
GRPO~\citep{shao_deepseekmath_2024} replaces the learned value network in PPO with a group-based Monte Carlo baseline. For each prompt $x\in\mathcal{X}$, a group of $K$ completions $y^{(1)}, \ldots, y^{(K)}\in\mathcal{Y}$ is sampled from the current policy $\pi_\phi(\cdot|x)$, and each receives a reward $r_j := r_\theta(x, y^{(j)})$. In the original GRPO, the advantage is the normalized reward: $\widehat{A}_j := (r_j - \mathrm{mean}(\mathbf{r})) / \mathrm{std}(\mathbf{r}),$ where $\mathbf{r}$ is the vectorized reward $(r_1,\ldots,r_K)^\mathsf{T},$ and $\mathrm{mean}(\cdot)$ and $\mathrm{std}(\cdot)$ denote the mean and standard deviation of the elements of a vector, respectively.

GRPO can be viewed as replacing PPO's learned value baseline with a within-group relative baseline computed from sampled completions. This removes the need to train a separate value network, allowing for preserving VRAM and scaling up the RLHF pipeline.

Dr.GRPO~\citep{liu_understanding_2025} is a widely-used variant that removes the standard-deviation normalization to avoid biased gradient estimation (see \S\ref{sec:grpo} in the main text).

\section{Proofs}\label{app:proofs}

In what follows, $h_\psi(x,y)$ is referred to as $h(x,y)$ for brevity, and $h_j$ denotes $h(x,y^{(j)}).$ We also write $A_j$ instead of $A_j(w)$ when the context is clear. We also sometimes omit the outer expectation $x\sim P(x).$

\begin{theorem}[Certified radius (Theorem~\ref{thm:robustness}, restated)]
	The certified sign-preservation radius for sample $j$ under head parameter uncertainty is
	\begin{equation*}
		\Delta_j^{\mathrm{param}} = \frac{|A_j|}{\|h_j - \bar{h}\|_2}.
	\end{equation*}
\end{theorem}

\begin{proof}[Proof of Theorem~\ref{thm:robustness}]
	We first establish how the advantage changes under head parameter perturbation. By linearity of the reward head,
	\begin{align*}
		A_j(w') & = (w')^\mathsf{T} h_j + b - \frac{1}{K}\sum_{k=1}^K \big((w')^\mathsf{T} h_k + b\big) \\
		        & = (w + \delta)^\mathsf{T} h_j - \frac{1}{K}\sum_{k=1}^K (w + \delta)^\mathsf{T} h_k   \\
		        & = A_j(w) + \delta^\mathsf{T}(h_j - \bar{h}),
	\end{align*}
	where $\bar{h} = \frac{1}{K}\sum_{k=1}^K h_k$. The worst-case perturbation for flipping the sign of $A_j > 0$ is
	$\delta^* = -\epsilon \cdot (h_j - \bar{h})/\|h_j - \bar{h}\|_2$,
	yielding $A_j(w + \delta^*) = A_j - \epsilon\|h_j - \bar{h}\|_2$. The sign flips when this becomes non-positive, i.e., when $\epsilon \geq |A_j|/\|h_j - \bar{h}\|_2$. The case $A_j < 0$ is symmetric.
\end{proof}

\begin{theorem}[Worst-case advantage under per-completion adversary (Theorem~\ref{lem:worst-case-advantage}, restated)]
	Under the linear head model and the uncertainty set $\mathcal{U}_\epsilon^w$~\eqref{eq:uncertainty-head}, the worst-case advantage~\eqref{eq:per-sample-adversary} for the $j$-th completion is
	\begin{equation*}
		\tilde{A}_j = \rho_j^* \cdot A_j, \quad \text{where} \quad
		\rho_j^* := 1 - \frac{\epsilon}{\Delta_j}
		= 1 - \frac{\epsilon\,\|h_j - \bar{h}\|_2}{|A_j|}.
	\end{equation*}
\end{theorem}

\begin{proof}[Proof of Theorem~\ref{lem:worst-case-advantage}]
	For $A_j > 0$, the worst case under $w' \in \mathcal{U}_\epsilon^w$ is $\inf_{w' \in \mathcal{U}_\epsilon^w} A_j(w') = A_j - \epsilon\|h_j - \bar{h}\|_2 = A_j(1 - \epsilon/\Delta_j)$, where the first equality follows from the perturbation analysis in the proof of Theorem~\ref{thm:robustness}. For $A_j < 0$, the analogous worst case is $\sup_{w' \in \mathcal{U}_\epsilon^w} A_j(w') = A_j + \epsilon\|h_j - \bar{h}\|_2 = A_j(1 - \epsilon/\Delta_j)$, where we used $|A_j| = -A_j$. In both cases, the worst-case advantage equals $\rho_j^* \cdot A_j$ with $\rho_j^*$ given by~\eqref{eq:optimal-weight}.
\end{proof}

\begin{proposition}[Head radius upper-bounds the full-parameter radius (Proposition~\ref{prop:head-bound}, restated)]
      Under the reward model $r_\theta(x,y)=w^\mathsf{T} h_\psi(x,y)+b$ with $\theta=(\psi,w,b)$,
      \begin{equation}
              \|\nabla_\theta A_j(\theta)\|_2^2 = \|h_j-\bar{h}\|_2^2 + \bigl\|(J_\psi h_j - J_\psi \bar{h})^\mathsf{T} w\bigr\|_2^2, \qquad J_\psi h := \frac{\partial h_\psi}{\partial \psi},
              \label{eq:grad-decomp}
      \end{equation}
      and consequently $\Delta_j^{\mathrm{head}} \ge \Delta_j^{\mathrm{full}}$.
\end{proposition}

\begin{proof}[Proof of Proposition~\ref{prop:head-bound}]
      Since the group mean subtracts the bias $b$, the advantage is $A_j(\theta) = w^\mathsf{T}\bigl(h_\psi(x,y^{(j)}) - \bar{h}\bigr)$ and does not depend on $b$. Differentiating,
      \begin{equation*}
              \nabla_w A_j = h_j - \bar{h}, \qquad \nabla_\psi A_j = (J_\psi h_j - J_\psi \bar{h})^\mathsf{T}\nabla_b A_j = 0.
      \end{equation*}
      Stacking these gives \eqref{eq:grad-decomp}. The second term is nonnegative, so $\|\nabla_\theta A_j\|\ge \|h_j - \bar{h}\|_2$, and dividing $|A_j|$ by the two norms reverses the inequality, giving $\Delta_j^{\mathrm{full}} = |A_j|/\|\nabla_\theta A_j\|_2 \le |A_j|/\|h_j-\bar{h}\|_2 = \Delta_j^{\mathrm{head}}$.
\end{proof}

\begin{lemma}[Policy gradient of the global robust objective (Lemma~\ref{lem:global-pg}, restated)]
	Under the GRPO sampling with $K$ on-policy completions per prompt, the PG of the global robust objective is
	\begin{equation*}
		\nabla_\phi J_{\mathrm{global}}(\phi,\theta)
		= \mathbb{E}_x \left[ \sum_{j=1}^{K} A_j \nabla_\phi \log \pi_\phi(y^{(j)} | x) \right]
		- \epsilon \nabla_\phi \bigl\| \mathbb{E}_{y \sim \pi_\phi}[h(x, y)] \bigr\|_2.
	\end{equation*}
\end{lemma}
\begin{proof}[Proof of Lemma~\ref{lem:global-pg}]
	We first derive the worst-case objective. Since the KL term is independent of $w'$, the infimum reduces to $\inf_{w' \in \mathcal{U}_\epsilon^w} \mathbb{E}_{y \sim \pi_\phi}[ {w'}^\mathsf{T} h(x,y) + b ]
		= w^\mathsf{T} \mathbb{E}_\pi[h] + b - \sup_{\|\delta\|_2 \le \epsilon} \delta^\mathsf{T} \mathbb{E}_\pi[h]
		= \mathbb{E}_\pi[r_w(x,y)] - \epsilon \| \mathbb{E}_\pi[h] \|_2$,
	where the supremum is attained at $\delta^* = \epsilon \mathbb{E}_\pi[h] / \| \mathbb{E}_\pi[h] \|_2$ by the Cauchy--Schwarz inequality. Thus,
	\begin{equation*}
		J_{\mathrm{global}}(\phi,w)
		= J(\phi, w)
		- \epsilon \bigl\| \mathbb{E}_{y \sim \pi_\phi}[h(x, y)] \bigr\|_2.
	\end{equation*}
	The first term of the PG is the standard GRPO gradient of $J(\phi, w)$.
	The second term follows from differentiating the penalty $-\epsilon \| \mathbb{E}_\pi[h] \|_2$ through the policy, which depends on $\phi$ via $\pi_\phi$ but not on any per-completion quantity.
\end{proof}

\section{Connection with randomized smoothing}\label{app:randomized-smoothing}

In \S\ref{sec:main-results}, we defined the \textit{certified sign-preservation radius} $\Delta_j$ \eqref{eq:cert-radius-def} as the largest perturbation under which the sign of the $j$-th completion's advantage remains constant. This section shows that $\Delta_j$ admits a formal interpretation through RS in the binary classification case.

\subsection{Randomized smoothing interpretation of the certified sign-preservation radius}

Consider the binary decision given by the sign of the advantage, and inject Gaussian noise $\xi\sim\mathcal{N}(0,\sigma^2I)$ into the space where uncertainty is modeled (e.g., the RM head parameter space). In RS, the certified radius for a binary decision is
\begin{equation*}
	R = \sigma\Phi^{-1}(p),
\end{equation*}
where $p$ is the probability that the decision is preserved under the Gaussian perturbation~\citep{cohen_certified_2019}.
For linear perturbations, the sign-preservation probability has the form $p = \Phi(|A_j|/(\sigma \cdot c))$ for some constant $c$. Substituting into the RS formula recovers exactly the certified radii in \eqref{eq:certified-radius-param} and \eqref{eq:feature-radius}.

The above results imply that responses with low coefficients $\rho_j^*$, or equivalently low radii $\Delta_j^\textrm{param}$, contribute less to the policy update with the robust policy optimization objective. The radius $\Delta_j^{\mathrm{param}} = |A_j|/\|h_j - \bar{h}\|_2$ is small when:
\begin{itemize}
	\item The advantage $|A_j|$ is small. In this case, the preference signal is weak and easily flipped.
	\item The feature deviation $\|h_j - \bar{h}\|_2$ is large. This means the response is an outlier in feature space, where the RM may be exploring.
\end{itemize}
The first category corresponds to the delta learning hypothesis~\citep{geng_delta_2025}. Our theoretical results suggest that reward hacking also exploits responses in the second category. Template-like or verbose outputs that achieve high rewards through spurious correlations often lie in sharp, high-norm feature directions, yielding small certified radii despite large advantages.

\subsection{Alternative uncertainty models}\label{sec:other-models}

The same derivation applies to the other uncertainty models, yielding analogous certified radii.

\paragraph{Feature uncertainty.}
Under perturbation $h_j' = h_j + \xi_j$ with $\xi_j \stackrel{\mathrm{i.i.d.}}{\sim} \mathcal{N}(0, \sigma_h^2 I)$, the perturbed advantage is
\begin{equation*}
	\tilde{A}_j = A_j + w^\top(\xi_j - \mathrm{mean}(\boldsymbol{\xi})).
\end{equation*}
Since $\xi_j - \mathrm{mean}(\boldsymbol{\xi}) \sim \mathcal{N}(0, \sigma_h^2(1 - 1/K)I)$, the perturbed advantage is Gaussian:
\begin{equation*}
	\tilde{A}_j \sim \mathcal{N}\left(A_j, \sigma_h^2(1 - 1/K)\|w\|_2^2\right).
\end{equation*}
The probability of sign preservation and certified radius are:
\begin{equation}
	p_j^{\mathrm{feat}} = \Phi\left(\frac{|A_j|}{\sigma_h\|w\|_2\sqrt{1 - 1/K}}\right), \qquad \Delta_j^{\mathrm{feat}} = \frac{|A_j|}{\|w\|_2}.
	\label{eq:feature-radius}
\end{equation}
Note that $\Delta_j^{\mathrm{feat}}$ is a monotone function of $|A_j|$, unlike $\Delta_j^{\mathrm{param}}$ which also depends on feature geometry.

\paragraph{Full parameter uncertainty.}
Under first-order approximation with $\theta' \sim \mathcal{N}(\theta, \sigma_\theta^2 I)$:
\begin{equation*}
	A_j(\theta') \approx A_j(\theta) + \nabla_\theta A_j(\theta)^\top(\theta' - \theta).
\end{equation*}
This yields:
\begin{equation*}
	p_j^{\mathrm{full}} = \Phi\left(\frac{|A_j|}{\sigma_\theta\|\nabla_\theta A_j\|_2}\right), \qquad \Delta_j^{\mathrm{full}} = \frac{|A_j|}{\|\nabla_\theta A_j\|_2}.
\end{equation*}
We expect spurious patterns to have large gradient norms $\|\nabla_\theta A_j\|_2$, as the RM assigns high rewards through sharp, non-robust computations. Such patterns receive small radii and are appropriately down-weighted.

\paragraph{On the head perturbation model.}
Although we made a simplifying assumption on our perturbation model in \S\ref{sec:method}, the head radius \eqref{eq:certified-radius-param} also controls the first-order full-parameter radius $\Delta_j^{\mathrm{full}}:=|A_j(\theta)|/\|\nabla_\theta A_j(\theta)\|_2$ from \S\ref{sec:main-results}. Writing $\Delta_j^\mathrm{head}$ for the head radius \eqref{eq:certified-radius-param}, the two are related as follows:
\begin{proposition}[Head radius upper-bounds the full-parameter radius]\label{prop:head-bound}
Under the reward model $r_\theta(x,y)=w^\mathsf{T} h_\psi(x,y)+b$ with $\theta=(\psi,w,b)$,
      \begin{equation}
              \|\nabla_\theta A_j(\theta)\|_2^2 = \|h_j-\bar{h}\|_2^2 + \bigl\|(J_\psi h_j - J_\psi \bar{h})^\mathsf{T} w\bigr\|_2^2, \qquad J_\psi h := \frac{\partial h_\psi}{\partial \psi},
      \end{equation}
      and consequently $\Delta_j^{\mathrm{head}} \ge \Delta_j^{\mathrm{full}}$.
\end{proposition}
The proof is given in Appendix~\ref{app:proofs}. The head radius thus upper-bounds the first-order full-parameter radius for any RM with a linear scalar head, whether the backbone $\psi$ was also trained or not. This does not by itself imply that the head radius orders completions the same way the full-parameter radius does. We examine that ordering empirically in \S\ref{sec:analysis}, where completions with larger $\Delta_j^\mathrm{head}$ also show higher sign robustness under full-parameter perturbations.

\paragraph{Summary of uncertainty models.}
Table~\ref{tab:uncertainty-models} summarizes the three models.

\begin{table}[t]
	\caption{Summary of uncertainty models and their certified radii.}
	\label{tab:uncertainty-models}
	\centering
	\begin{tabular}{lccc}
		\toprule
		\textbf{Model} & \textbf{Uncertainty set}           & \textbf{Certified radius} $\Delta_j$ & \textbf{Interpretation} \\
		\midrule
		Head param.    & $\|w' - w\|_2 \leq \epsilon$       & $|A_j|/\|h_j - \bar{h}\|_2$          & Feature geometry        \\
		Feature        & $\|h' - h\|_2 \leq \delta$         & $|A_j|/\|w\|_2$                      & Advantage magnitude     \\
		Full param.    & $\|\theta' - \theta\|_2 \leq \tau$ & $|A_j|/\|\nabla_\theta A_j\|_2$      & Gradient sharpness      \\
		\bottomrule
	\end{tabular}
\end{table}

\subsection{Relation to a Bayesian treatment of the head}\label{app:laplace}

A probabilistic alternative to the uncertainty set $\mathcal{U}_\epsilon^w$~\eqref{eq:uncertainty-head} is a posterior distribution over the head $w$. We consider the last-layer Laplace approximation~\citep{daxberger_laplace_2021} of the BT objective~\eqref{eq:bt-loss}. Let $D=\{(x_i,y_{\mathrm{w},i},y_{\mathrm{l},i})\}_{i=1}^N$ be the RM training dataset. For the $i$-th triple, let $u_i := h_\psi(x_i,y_{\mathrm{w},i}) - h_\psi(x_i,y_{\mathrm{l},i}) \in \mathbb{R}^d$ be the feature difference and $m_i := w^\mathsf{T} u_i$ its reward margin. With a Gaussian prior of precision $\lambda > 0$ on the head, the Laplace approximation of the posterior is $\mathcal{N}(w, H^{-1})$, where
\begin{equation*}
	H := \sum_{i=1}^N \sigma(m_i)\bigl(1-\sigma(m_i)\bigr)\, u_i u_i^\mathsf{T} + \lambda I_d \in \mathbb{R}^{d\times d}.
\end{equation*}
Since the advantage is linear in the head, $A_j(w') = {w'}^\mathsf{T}(h_j - \bar{h})$ is Gaussian for $w' \sim \mathcal{N}(w, H^{-1})$, with mean $A_j(w)$ and standard deviation $s_j := \|H^{-1/2}(h_j - \bar{h})\|_2$. The posterior probability of a sign flip is therefore $\Phi(-z_j)$, where $z_j := |A_j(w)|/s_j$ and $\Phi$ is the standard normal distribution function. Two consequences follow.

\paragraph{The posterior mean recovers Dr.GRPO.}
Since $\mathbb{E}[A_j(w')] = A_j(w)$, replacing each advantage with its posterior mean recovers the Dr.GRPO update.

\paragraph{$z_j$ is a certified radius in the posterior metric.}
Replace the Euclidean ball $\mathcal{U}_\epsilon^w$ with the ellipsoid $\{w' \in \mathbb{R}^d : (w'-w)^\mathsf{T} H (w'-w) \le \epsilon^2\}$. The derivation of Theorem~\ref{thm:robustness} then gives $\Delta_j = z_j$, and Theorem~\ref{lem:worst-case-advantage} holds with $\rho_j^* = 1 - \epsilon/z_j$.

\method{} and this Bayesian counterpart therefore differ in two choices: the geometry of the uncertainty set (a Euclidean ball or an ellipsoid defined by $H$) and how the uncertainty is aggregated (the worst case or the probability of a sign flip). Computing $H$ requires the RM training data, which \method{} does not use.

\section{Robust optimization perspective and its limitations}\label{app:per-sample-adversary}

The re-weighted gradient~\eqref{eq:conservative-pg} can be viewed as the PG of a per-completion robust optimization problem. For each completion, the adversary chooses $w' \in \mathcal{U}_\epsilon^w$ to maximally corrupt that completion's gradient contribution:
\begin{equation}\label{eq:per-sample-robust}
	J_{\mathrm{robust}}(\phi) = \mathbb{E}_{x \sim P} \left[ \sum_{j=1}^{K} \inf_{w' \in \mathcal{U}_\epsilon^w} A_j(w') \cdot \mathbf{1}[A_j(w) > 0] + \sup_{w' \in \mathcal{U}_\epsilon^w} A_j(w') \cdot \mathbf{1}[A_j(w) < 0] \right].
\end{equation}
This formulation takes the $\inf$ and $\sup$ inside the sum over completions, allowing the adversary to choose different perturbations for each completion.

We now argue that this per-sample decomposition is a necessity: any formulation that robustifies an aggregate advantage is vacuous under GRPO's mean-subtracted structure.

\paragraph{Zero-sum obstruction.}
A more standard robust optimization approach~\citep{ben-tal_robust_2009} would robustify the aggregate advantage directly, e.g., $\inf_{w' \in \mathcal{U}_\epsilon^w} \sum_{j=1}^{K} A_j(w')$. However, under the linear head model, for any $w' \in \mathbb{R}^d$,
\begin{equation*}
	\sum_{j=1}^{K} A_j(w') = (w')^\mathsf{T} \sum_{j=1}^{K} (h_j - \bar{h}) = 0,
\end{equation*}
since $\sum_{j=1}^K (h_j - \bar{h}) = 0$ by definition of $\bar{h}$. Group-relative advantages sum to zero regardless of the RM parameters, so any robust objective of the form $\inf_{w'} \sum_j A_j(w')$ evaluates to zero trivially and provides no useful signal for policy optimization.

This complements Lemma~\ref{lem:global-pg}, which showed that the global robust reward objective yields only a completion-agnostic penalty~\eqref{eq:global-pg}.

\paragraph{Conservatism of the per-sample formulation.}
The per-sample objective~\eqref{eq:per-sample-robust} circumvents the zero-sum obstruction by partitioning the sum according to the sign of $A_j(w)$ before applying the $\inf$ and $\sup$. This is a stronger (i.e., more conservative) setting than the standard robust optimization formulation, where a single worst-case parameter $w'$ would be shared across all terms:
\begin{equation}\label{eq:shared-robust}
	J_{\mathrm{robust}}^{\mathrm{shared}}(\phi) = \mathbb{E}_{x \sim P} \left[ \inf_{w' \in \mathcal{U}_\epsilon^w} \sum_{j=1}^{K} A_j(w') \cdot \mathbf{1}[A_j(w) > 0] + \sup_{w' \in \mathcal{U}_\epsilon^w} \sum_{j=1}^{K} A_j(w') \cdot \mathbf{1}[A_j(w) < 0] \right].
\end{equation}
The per-sample formulation provides a tractable upper bound on the conservatism: because each term is optimized independently, the resulting coefficients $\rho_j^*$ decompose into per-sample quantities, yielding the closed-form solution in Theorem~\ref{thm:robustness}. The shared-adversary variant~\eqref{eq:shared-robust}, while tighter, couples all completions and requires solving a joint optimization problem over the group. The further analysis of different robust optimizations is left for future work.

\section{Extension to other policy gradient methods}\label{app:extension-pg}

The sign-certification principle of \S\ref{sec:method} applies beyond Dr.GRPO to any PG-based RLHF method whose update involves a scalar coefficient depending on the RM. Since the exact form of the certified radius depends on the update geometry, we only discuss REINFORCE~\citep{williams_simple_1992}, PPO~\citep{schulman_proximal_2017} and vanilla GRPO~\citep{shao_deepseekmath_2024}.

\subsection{REINFORCE}\label{app:reinforce-extension}

The simplest case is REINFORCE, or more generally any response-level PG method with advantage
\begin{equation*}
	A^{\mathrm{R}}(x,y;w) = r_w(x,y) - V(x),
\end{equation*}
where $r_w(x,y) = w^\mathsf{T} h(x,y) + b$ is the linear reward head and $V(x)$ is a baseline independent of the RM head $w$ (e.g., a learned value function, a Monte Carlo return, or simply $V(x) \equiv 0$ for vanilla REINFORCE). The PG then takes the form $\nabla_\phi J(\phi) = \mathbb{E}_{x,y}[A^{\mathrm{R}}(x,y;w)\nabla_\phi \log \pi_\phi(y|x)]$, matching the form in~\eqref{eq:policy-gradient}.

Under the same linear-head uncertainty set $\mathcal{U}_\epsilon^w$~\eqref{eq:uncertainty-head} as in \S\ref{sec:method}, the certified sign-preservation radius admits an exact closed form.

\begin{proposition}[Certified sign-preservation radius for response-level PG]\label{prop:reinforce-radius}
	Under the linear reward head and the uncertainty set $\mathcal{U}_\epsilon^w$~\eqref{eq:uncertainty-head}, with a baseline $V(x)$ independent of the RM head $w$, the certified sign-preservation radius for response $y$ is
	\begin{equation}
		\Delta^{\mathrm{R}}(x,y) = \frac{|A^{\mathrm{R}}(x,y;w)|}{\|h(x,y)\|_2}.
		\label{eq:reinforce-radius}
	\end{equation}
\end{proposition}

\begin{proof}
	Under perturbation $w' = w + \delta$ with $\|\delta\|_2 \leq \epsilon$:
	\begin{equation*}
		A^{\mathrm{R}}(x,y;w') = A^{\mathrm{R}}(x,y;w) + \delta^\top h(x,y),
	\end{equation*}
	since $V(x)$ does not depend on $w$. By the Cauchy--Schwarz inequality, $|\delta^\top h(x,y)| \leq \epsilon\|h(x,y)\|_2$, with equality when $\delta = \pm\epsilon\, h(x,y)/\|h(x,y)\|_2$. The sign flips when $\epsilon \geq |A^{\mathrm{R}}(x,y;w)|/\|h(x,y)\|_2$.
\end{proof}

We can also derive the worst-case advantage in the vein of Theorem~\ref{lem:worst-case-advantage}: 
\begin{corollary}[Worst-case advantage under per-completion adversary]\label{cor:reinforce-reweighting}
	Under the same setting as Proposition~\ref{prop:reinforce-radius}, the worst-case advantage under the per-completion adversary~\eqref{eq:per-sample-adversary} is
	\begin{equation*}
		\tilde{A}^{\mathrm{R}}(x,y) = \rho^{\mathrm{R}} \cdot A^{\mathrm{R}}(x,y;w), \quad \text{where} \quad \rho^{\mathrm{R}} := 1 - \frac{\epsilon}{\Delta^{\mathrm{R}}(x,y)}
		= 1 - \frac{\epsilon\,\|h(x,y)\|_2}{|A^{\mathrm{R}}(x,y;w)|}.
	\end{equation*}
\end{corollary}

\subsection{Approximate extensions}\label{app:approximate-extensions}

Standard PPO and vanilla GRPO do not generally admit a single response-level coefficient whose sign alone determines the update for the whole completion. Sign certification for these methods therefore requires a first-order approximation.

\subsubsection{PPO}\label{app:ppo-extension}

Standard PPO in RLHF uses a token-level clipped surrogate objective~\eqref{eq:ppo-objective} with token-level advantages $\hat{A}_t$~\eqref{eq:ppo-advantage} estimated via a learned value network (\S\ref{app:pg-algorithms}). Unlike Dr.GRPO or the response-level PG of \S\ref{app:reinforce-extension}, PPO does not generally admit a single scalar response-level coefficient. Its per-token structure means that the sign of the effective coefficient can differ across tokens within a single response.

Sign certification for PPO should therefore be formulated for the token-level coefficient appearing in each surrogate term. Under the linear reward head, the token-level advantage $\hat{A}_t$ depends on the RM head $w$ only through the terminal reward $r_w(x,y)$ added at $t = T$. A first-order approximation gives
\begin{equation*}
	\Delta_t^{\mathrm{PPO}} \approx \frac{|\hat{A}_t|}{\|\nabla_w \hat{A}_t\|_2},
\end{equation*}
which is consistent with the general first-order extension in \S\ref{sec:main-results}.

\subsubsection{Vanilla GRPO}\label{app:vanilla-grpo}

In vanilla GRPO~\citep{shao_deepseekmath_2024}, the advantage is the normalized reward $\hat{A}_j = (r_j - \mathrm{mean}(\mathbf{r}))/\mathrm{std}(\mathbf{r})$, where $\mathbf{r} = (r_1,\ldots,r_K)^\mathsf{T}$. Because $\mathrm{std}(\mathbf{r}(w'))$ depends non-linearly on $\delta = w' - w$ through all $K$ rewards, the perturbed advantage $\hat{A}_j(w')$ is no longer affine in $\delta$. This breaks the exact analysis that yields the Dr.GRPO closed form~\eqref{eq:certified-radius-param}. A first-order expansion yields the approximate radius
\begin{equation*}
	\Delta_j^{\mathrm{GRPO}} \approx \frac{|\hat{A}_j|}{\|\nabla_w \hat{A}_j\|_2},
\end{equation*}
which requires the gradient of the normalized advantage with respect to $w$. This is more expensive than the Dr.GRPO closed form but avoids the full-parameter gradients needed for the general RM extension (\S\ref{sec:main-results}).

\section{Training and implementation details}\label{app:algorithm}

\subsection{Reproducibility statement}

All experiments are conducted on a shared compute cluster. Each training job uses 8 NVIDIA GH200 nodes, each equipped with an NVIDIA H100 GPU. Response generation during the gold-model pipeline uses vLLM for efficient batched inference. Complete hyperparameter settings for SFT, RM training and policy optimization are listed in Appendix~\ref{app:hyperparameters} (Tables~\ref{tab:hp-sft}--\ref{tab:hp-policy}). We use the publicly available dataset\footnote{\url{https://huggingface.co/datasets/vwxyzjn/summarize_from_feedback_tldr_3_filtered}} for our TL;DR experiments. For BSPO and AdvPO, we did not find the public implementation and thus we implemented them ourselves, based on the respective papers. For the 1-epoch proxy RM experiments on Pythia-1B and Qwen2.5-1.5B, we report mean $\pm$ standard deviation over 3 random seeds, corresponding to different SFT models, sampled preference datasets, proxy RMs and resulting policies. For the other settings in Table~\ref{tab:main-results}, Dr.GRPO, \method{}, and AdvPO on AlpacaFarm also use 3 random seeds. Our implementation is available at \url{https://github.com/shinnosukeono/SignCert-PO}.

\subsection{Conceptual experiment (Figure~\ref{fig:overview})}\label{app:toy-bandit}

Before presenting the full-scale LLM experiments, we design a controlled conceptual experiment to demonstrate advantage sign flipping causes reward hacking and to verify that \method{} prevents it.

\subsubsection{Setup}

We consider a stateless $M$-armed bandit ($M=8$). Each action (completion) $m$ is associated with a fixed feature vector $h_m\in\mathbb{R}^d$ ($d=6$), analogous to the last-layer hidden states $h(x,y^{(j)})$ of an RM backbone in \S\ref{sec:setup}. Both the proxy and true RMs use a linear head on top of these shared features:
\begin{equation*}
	r_k = w^\mathsf{T} h_k, \qquad r^*_k = w^{*\mathsf{T}} h_k,
\end{equation*}
matching the linear-head architecture assumed in \S\ref{sec:setup}. The feature space $\mathbb{R}^d$ is split into two subspaces: \emph{shared} dimensions (0--2), where both heads assign weight $1.0$; \emph{divergent} dimensions (3--5), where the proxy assigns weight $0.4$ but the true RM assigns $-0.1$. Actions 1--3 (\emph{reliable positive}) use only shared dimensions and have positive proxy and true rewards. Actions 4--6 (\emph{reliable negative}) also use only shared dimensions but have negative rewards. Actions 7--8 (\emph{hacking}) use only divergent dimensions. As a result, the proxy RM assigns them the highest advantages in the group, while their true advantages are negative.

The policy is a categorical stochastic policy $\pi(a_m|\phi)=\mathrm{softmax}(\phi)_m$ initialized to uniform ($\phi=\mathbf{0}$). Because $M$ is small and there is a single state, we use the exact expected policy gradient.

\subsubsection{Policy optimization methods}

We compare three methods, all sharing the softmax policy gradient structure but differing in how advantages are weighted:

\begin{itemize}
	\item \textbf{Dr.GRPO}: $\nabla_{\phi_i} J = \pi_i(A_i - V_\pi)$, where $V_\pi = \sum_k \pi_k A_k$.
	\item \textbf{Global robust objective} (Lemma~\ref{lem:global-pg}): adds a penalty $-\epsilon\pi_i(h_i - \mathbb{E}_\pi[h])^\mathsf{T}(\mathbb{E}_\pi[h]/\|\mathbb{E}_\pi[h]\|)$ that corrects all actions through a single direction. We sweep $\epsilon$ over 20 log-spaced values in $[10^{-2}, 10]$ and report the best.
	\item \textbf{\method{}} (Theorem~\ref{thm:robustness}): $\nabla_{\phi_i} J = \pi_i(\rho_i A_i - V_\rho)$, where $\rho_k = 1 - \epsilon/\Delta_k$ and $\epsilon$ is set to the $q_t$-th quantile of $\{\Delta_k\}$.
\end{itemize}

\subsubsection{Results}

Figure~\ref{fig:overview}(a) shows the proxy advantages per action. Red bars indicate completions whose advantage sign disagrees with the true RM. The hacking actions (7--8) have the largest proxy advantages. Figure~\ref{fig:overview}(c) shows the per-action certified sign-preservation radius $\Delta_j$. The completions 7 and 8  cluster at low $\Delta$ values because they are outliers in the feature-space (i.e., large $\|h_j - \bar{h}\|$), confirming that $\Delta_j$ identifies unreliable completions.

Figure~\ref{fig:overview}(d) tracks true (solid) and proxy (dashed) expected reward during policy optimization. Dr.GRPO converges on the hacking actions, where true reward collapses. The global robust penalty partially mitigates reward hacking, but the single penalty direction cannot distinguish per-action reliability and is thus too conservative.  \method{} achieves the highest true reward, by suppressing the hacking actions' gradient contributions via their small $\rho_k$ weights.

\subsection{Algorithm pseudocode}

Algorithm~\ref{alg:rapo} implements the \method{} re-weighting of GRPO derived in \S\ref{sec:algorithm}. We highlight the modifications over standard Dr.GRPO in \colorbox{blue!8}{blue}.

\begin{algorithm}[t]
	\caption{\method{} Algorithm.}
	\label{alg:rapo}
	\begin{algorithmic}[1]
		\Require Policy $\pi_\phi$, reward model $r_w(h)=w^\top h + b$, quantile threshold $q_t$, group size $K$
		\For{each prompt $x$ in batch}
		\State Sample $K$ completions: $y^{(1)}, \ldots, y^{(K)} \sim \pi_\phi(\cdot|x)$
		\State Compute $h_j=h(x,y^{(j)})$, $r_j=w^\mathsf{T} h_j+b$, $A_j=r_j-\frac{1}{K}\sum_{k=1}^Kr_k$ for $j\in[K]$
		\BeginBox[fill=blue!8, draw=blue!40, rounded corners]
		\State Compute $\bar{h}=\frac{1}{K}\sum_{k=1}^Kh_k$, \ $\Delta_j=|A_j|/\|h_j-\bar{h}\|_2$
		\EndBox
		\EndFor
		\State Set $\epsilon$ to the $q_t$-th quantile of $\{\Delta_j\}$ over all completions in the batch
		\For{each completion $j$}
		\BeginBox[fill=blue!8, draw=blue!40, rounded corners]
		\State $\rho_j^* \gets 1 - \epsilon / \Delta_j$ \Comment{stop-gradient on $A_j, \Delta_j$}
		\State $\widetilde{A}_j \gets \rho_j^* \cdot A_j$
		\EndBox
		\EndFor
		\State Update policy with PG using $\{\widetilde{A}_j\}$
	\end{algorithmic}
\end{algorithm}

\subsection{Training details}

\subsubsection{Datasets}\label{sec:dataset}

\paragraph{TL;DR summarization.}
We use the filtered version\footnote{\url{https://huggingface.co/datasets/vwxyzjn/summarize_from_feedback_tldr_3_filtered}} of OpenAI's Summarize from Feedback dataset~\citep{stiennon_learning_2020}, with 116K training and 6K evaluation prompts. Each prompt is a Reddit post paired with a human-written reference summary.

\paragraph{AlpacaFarm.}
We additionally evaluate on the AlpacaFarm benchmark~\citep{dubois_alpacafarm_2023}, a simulation framework for learning from human feedback built on 52K instruction-following samples. We use 10K prompts for training and 2K for evaluation, and take the output field of the dataset as the reference response.

\subsubsection{Models}

\paragraph{Policy and proxy RM.}
We use Pythia~\citep{biderman_pythia_2023} and Qwen2.5~\citep{qwen_qwen25_2025} series as policy and proxy RM base models. We vary the policy size among Pythia 1B and 2.8B, and Qwen2.5 1.5B and 3B. The same model size as the policy is chosen for the proxy RM.

For TL;DR, the proxy RM is trained with 1 epoch of preference data. For AlpacaFarm, the proxy RM is trained with 3 epochs of preference data, roughly matching the number of queries for TL;DR (see \S\ref{sec:dataset}).
\paragraph{Gold RM.}
For the TL;DR experiments with Pythia 1B and Qwen2.5 1.5B policies, the gold reward model is Skywork-Reward-Llama-3.1-8B-v0.2~\citep{liu_skywork-reward_2024}, one of the top-performing discriminative reward models on RewardBench~\citep{malik_rewardbench_2025}. All preference data and evaluation metrics for these settings are computed using this gold RM. For larger policy models (Pythia 2.8B and Qwen2.5 3B), we use GPT-4.1 Nano~\citep{openai_gpt-4_2024} as the gold RM, since a model with stronger capability than the proxy RM is required.

\subsubsection{Baselines}\label{sec:baselines}

We compare \method{} (Algorithm~\ref{alg:rapo}) against the following baselines:
\begin{itemize}
	\item \textbf{SFT}: the SFT policy $\pi_\mathrm{SFT}$ without any RL optimization.
	\item \textbf{Dr.GRPO}~\citep{liu_understanding_2025}, as introduced in \S\ref{sec:grpo}. The KL coefficient is swept over $\{0, 0.001, 0.005, 0.01, 0.1, 0.2\}$ for Pythia-1B and Qwen2.5-1.5B, and the best configuration is selected per dataset. Due to computational constraints, we do not sweep the KL coefficient for larger models and instead use 0.005 for Pythia 2.8B and 0.01 for Qwen2.5 3B.
	\item \textbf{UWO}~\citep{coste_reward_2023}: an RM ensemble method that trains multiple proxy RMs with different random seeds and uses the mean reward minus a penalty proportional to the ensemble standard deviation. We use an ensemble of three RM members. Following the setup of \citet{zhang_mitigating_2024}, to control total model capacity, we use smaller individual RMs (Pythia 410M for 1B experiments and Qwen2.5 0.5B for 1.5B experiments). For larger policy models (Pythia 2.8B, Qwen2.5 3B), we use Pythia 1B and Qwen2.5 1.5B as proxy RMs in the ensemble, respectively.
	\item \textbf{BSPO}~\citep{dai_mitigating_2024}: a method that constrains policy updates to remain within the support of the behavior policy to prevent out-of-distribution exploitation. We sweep the behavior-support threshold $\epsilon_\beta \in \{10^{-4}, 10^{-5}, 10^{-6}\}$ and select the best configuration per dataset.
	\item \textbf{AdvPO}~\citep{zhang_mitigating_2024}: an adversarial policy optimization method that augments the reward with a learned uncertainty penalty. We sweep the perturbation budget $B \in \{1, 5, 10, 15\}$, following the range used in the original paper.
\end{itemize}

For both benchmarks, we set the learning rate to $10^{-6}$, with micro batch size 32 on 8 GPUs and $K=8$ generations per prompt.
Pythia policies are trained for 2,000 steps; Qwen2.5 for 500 steps since proxy reward converges faster.
For \method{}, we sweep the quantile threshold $q_t \in \{0.1, 0.2, 0.25, 0.3, 0.4, 0.5\}$.
For every method, we select the hyperparameters by the gold-RM win rate on the evaluation prompts.
For TL;DR experiments with Pythia 1B and Qwen2.5 1.5B, we report mean $\pm$ standard deviation over 3 random seeds to account for variance from limited preference data.
For the other settings, Dr.GRPO, \method{}, and AdvPO on AlpacaFarm also use 3 random seeds, and the remaining entries in Table~\ref{tab:main-results} use a single seed.

\subsubsection{Evaluation metrics}\label{sec:metrics}

We report the win rate of policy outputs against reference responses, judged by the gold RM. We also measure proxy RM accuracy and KL divergence during policy optimization, measured at every 50 global optimization steps. Proxy RM accuracy is calculated as agreement between the proxy RM and the gold RM on pairwise ranking comparisons of on-policy completions.

\subsection{Hyperparameters}\label{app:hyperparameters}

Tables~\ref{tab:hp-sft}--\ref{tab:hp-policy} list all hyperparameters used in our experiments.

\begin{table}[h]
\caption{SFT and preference data generation hyperparameters.}
	\label{tab:hp-sft}
	\centering
	\small
	\begin{tabular}{lcccc}
		\toprule
		Hyperparameter              & Pythia 1B          & Pythia 2.8B        & Qwen2.5 1.5B       & Qwen2.5 3B         \\
		\midrule
		\multicolumn{5}{l}{\emph{SFT}} \\
		Epochs                      & 1                  & 1                  & 1                  & 1                  \\
		Learning rate               & $3\times10^{-6}$   & $3\times10^{-6}$   & $3\times10^{-6}$   & $3\times10^{-6}$   \\
		LR scheduler                & cosine             & cosine             & cosine             & cosine             \\
		Per-device batch size       & 16                 & 16                 & 16                 & 16                 \\
		bf16                        & \checkmark         & \checkmark         & \checkmark         & \checkmark         \\
		\midrule
		\multicolumn{5}{l}{\emph{Preference data generation}} \\
		Completions per prompt      & 2                  & 2                  & 2                  & 2                  \\
		Sampling temperature        & 1.0                & 1.0                & 1.0                & 1.0                \\
		Max new tokens              & 53                 & 53                 & 63                 & 63                 \\
		Gold RM                     & Skywork-8B         & GPT-4.1 Nano       & Skywork-8B          & GPT-4.1 Nano      \\
		\bottomrule
	\end{tabular}
\end{table}

\begin{table}[h]
\caption{Proxy reward model training hyperparameters. The same base model architecture as the SFT policy is used with a randomly initialized scalar head.}
	\label{tab:hp-rm}
	\centering
	\small
	\begin{tabular}{lcc}
		\toprule
		Hyperparameter              & Pythia 1B                                 & Qwen2.5 1.5B     \\
		\midrule
		Preference data epochs      & \multicolumn{2}{c}{$\{1, 3, 5, 10, 20\}$}                    \\
		Max length                  & 512                                       & 512              \\
		Learning rate               & $2\times10^{-5}$                          & $2\times10^{-5}$ \\
		LR scheduler                & cosine                                    & cosine           \\
		Warmup ratio                & 0.03                                      & 0.03             \\
		Weight decay                & 0.1                                       & 0.1              \\
		Per-device batch size       & 16                                        & 16               \\
		Gradient accumulation steps & 1                                         & 1                \\
		Gradient checkpointing      & \checkmark                                & \checkmark       \\
		bf16                        & \checkmark                                & \checkmark       \\
		\bottomrule
	\end{tabular}
\end{table}

\begin{table}[h]
\caption{Policy optimization hyperparameters (Dr.GRPO / \method{}). For Pythia 1B and Qwen2.5 1.5B, the KL coefficient is selected via sweep over $\{0, 0.001, 0.005, 0.01, 0.1, 0.2\}$ (c.f., Appendix~\ref{sec:baselines}). For the RM epoch ablation (Appendix~\ref{app:ablation}), the KL coefficient is re-tuned per epoch.}
	\label{tab:hp-policy}
	\centering
	\small
	\begin{tabular}{lcccc}
		\toprule
		Hyperparameter             & Pythia 1B                                               & Pythia 2.8B      & Qwen2.5 1.5B     & Qwen2.5 3B       \\
		\midrule
		Training steps             & 2{,}000                                                 & 2{,}000          & 500              & 500              \\
		Generations per prompt $K$ & 8                                                       & 8                & 8                & 8                \\
		Max new tokens             & 53                                                      & 53               & 63               & 63               \\
		Temperature                & 1.0                                                     & 1.0              & 1.0              & 1.0              \\
		Per-device batch size      & 2                                                       & 2                & 2                & 2                \\
		Gradient accumulation steps & 16                                                     & 16               & 16               & 16               \\
		Number of GPUs             & 8                                                       & 8                & 8                & 8                \\
		Learning rate              & $1\times10^{-6}$                                        & $1\times10^{-6}$ & $1\times10^{-6}$ & $1\times10^{-6}$ \\
		LR scheduler               & constant                                                & constant         & constant         & constant         \\
		Clip range ratio           & 0.2                                                     & 0.2              & 0.2              & 0.2              \\
		KL coefficient $\beta$     & 0.01                                                    & 0.005            & 0.1              & 0.01             \\
		\method{} quantile $q_t$   & \multicolumn{4}{c}{$\{0.1, 0.2, 0.25, 0.3, 0.4, 0.5\}$}                                                          \\
		bf16                       & \checkmark                                              & \checkmark       & \checkmark       & \checkmark       \\
		\bottomrule
	\end{tabular}
\end{table}

\subsection{Other experimental details}\label{app:other-experimental-details}

\paragraph{Certified radius validation (Figure~\ref{fig:radius_validation}).}
We test whether the certified sign-preservation radius $\Delta_j$ \eqref{eq:certified-radius-param}, derived under a linear head perturbation assumption, predicts robustness under broader perturbation classes.
We use a Pythia 1B SFT policy to generate $K{=}8$ completions (53 tokens each, temperature 0.7) for 500 prompts from the TL;DR validation split. The Pythia 1B proxy RM scores each completion and provides last-layer hidden states $h_j \in \mathbb{R}^{2048}$. For the gold RM we use Skywork-Reward-Llama-3.1-8B-v0.2, the one used for RM training. Groups with reward standard deviation below $10^{-4}$ are filtered.
We consider three settings beyond the analytically tractable head-only case.
\emph{Whole-model perturbation}: for each parameter tensor~$p$, we add $\delta_p \sim \mathcal{N}(0, (\sigma \|p\|_F / \sqrt{\mathrm{numel}(p)})^2 I)$ with $\sigma=0.1$ and draw 10 samples.
\emph{Input embedding perturbation}: isotropic Gaussian noise $\delta \sim \mathcal{N}(0, \sigma_{\mathrm{in}}^2 I)$ is added to the embedding layer output before the transformer backbone, with $\sigma_{\mathrm{in}}=0.01$ and 10 samples.
\emph{Gold RM agreement}: we compare the advantage signs of the proxy and gold RMs without any perturbation.
Completions are binned into 10 quantile groups (deciles) by $\Delta_j$. For each bin we report the mean sign-preservation rate or gold agreement rate with 95\% confidence intervals.

\section{Additional results}\label{app:additional-results}

\subsection{Additional comparison with baselines}\label{app:additional-baselines}

\paragraph{Stronger KL penalty.}
Figure~\ref{fig:kl_ablation_appendix} sweeps the KL coefficient $\beta$ for Dr.GRPO on Pythia 1B, with the \method{} trajectory. Increasing $\beta$ does constrain the policy to lower KL regions, but this trades off exploration for safety.
\method{} reaches the highest peak gold reward at a moderate KL budget.
We argue that the KL penalty controls \emph{how far} the policy moves from the reference distribution, uniformly penalizing all deviations. On the other hand, \method{} controls \emph{which directions} the policy moves, selectively suppressing updates where the RM is unreliable.

\begin{figure}[h]
	\centering
	\includegraphics[width=0.48\textwidth]{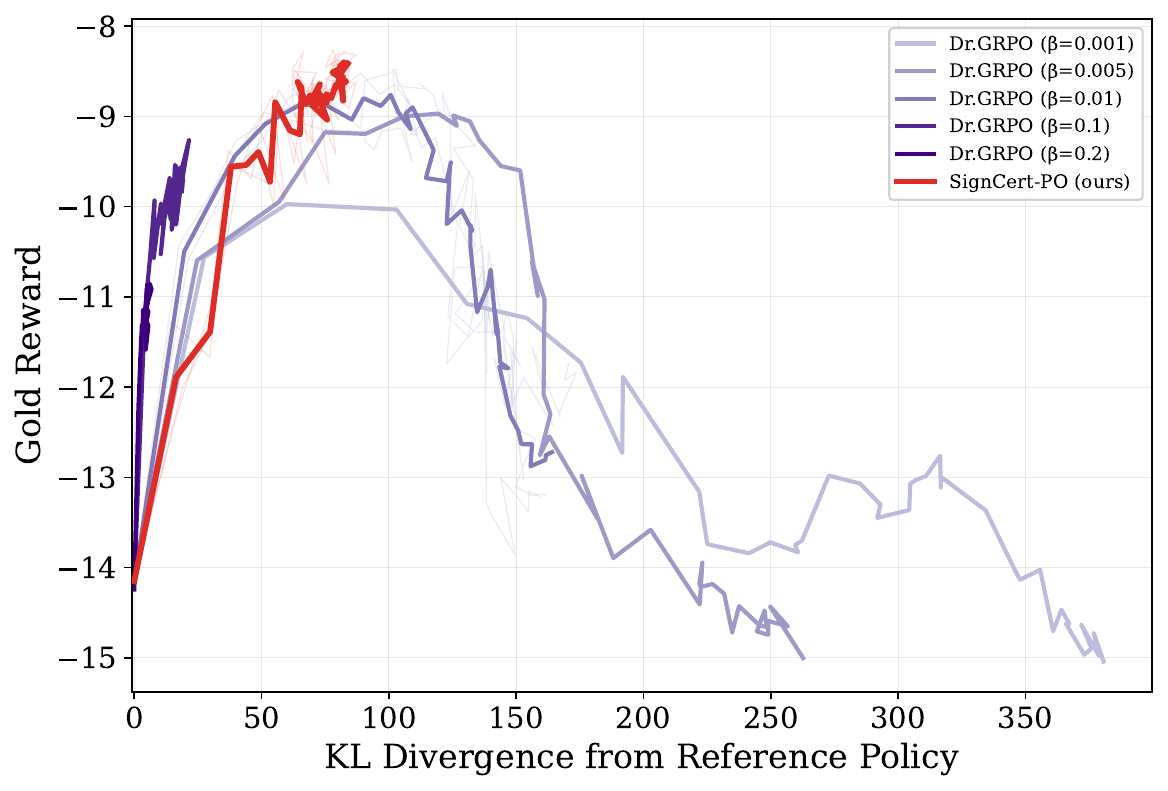}
	\caption{\textbf{Increasing $\beta$ constrains the policy to lower KL regions but trades off exploration for safety.} KL coefficient sweep for Dr.GRPO on Pythia 1B and TL;DR with the \method{} trajectory.}
	\label{fig:kl_ablation_appendix}
\end{figure}

\paragraph{UWO with same-size RMs.}
\begin{table}[h]
	\caption{UWO RM size ablation on TL;DR (Pythia 1B).}
	\label{tab:uwo-rm-size}
\centering
\begin{tabular}{lcc}
	\small
	\begin{tabular}{lcc}
		\toprule
		Method           & RM size & Win rate        \\
		\midrule
		UWO ($\times 3$) & 410M$\times 3$    & $6.1 \pm 3.6$  \\
		UWO ($\times 3$) & 1B$\times 3$      & 5.9             \\
		\midrule
		\method{}        & 1B      & $\mathbf{60.0 \pm 2.0}$ \\
		\bottomrule
	\end{tabular}
\end{tabular}
\end{table}

In our experiments (Table~\ref{tab:main-results}), we followed \citet{zhang_mitigating_2024} and use smaller individual RMs for UWO to control total capacity (Appendix~\ref{sec:baselines}). As an ablation, we train an ensemble of three Pythia 1B RMs with different random seeds and run UWO under the same setting on TL;DR with the Pythia 1B policy. As shown in Table~\ref{tab:uwo-rm-size}, increasing the individual RM size does not improve UWO.

\subsection{Combination with off-policy corrected reward modeling}\label{app:ocrm}
Off-policy corrected reward modeling (OCRM)~\citep{ackermann_off-policy_2025} retrains the RM on on-policy completions so that it stays accurate as the policy drifts. Because \method{} acts only at the policy optimization stage, it applies on top of OCRM without modification. Table~\ref{tab:ocrm} reports win rates on TL;DR with Pythia 1B, retraining the RM every 100K samples over 300K samples of training. \method{} improves over Dr.GRPO with and without OCRM, and \method{} on its own already matches Dr.GRPO combined with OCRM. This is consistent with the two operating at different stages of the pipeline.

\begin{table}[h]
	\caption{Win rate (\%) on TL;DR with Pythia 1B, with and without OCRM. OCRM uses a single seed.}
	\label{tab:ocrm}
    \centering
	\small
	\begin{tabular}{lc}
		\toprule
		Method & Win rate \\
		\midrule
		Dr.GRPO            & $21.0 \pm 5.4$ \\
		Dr.GRPO + OCRM     & $57.6$ \\
		\method{}          & $60.0 \pm 2.0$ \\
		\method{} + OCRM   & $\mathbf{62.3}$ \\
		\bottomrule
	\end{tabular}
\end{table}

\subsection{Sensitivity to the gold RM}\label{app:gold-rm}
Our evaluation uses a gold RM as a proxy for true human preference, so the comparison could in principle depend on the specific gold RM. We re-evaluate the TL;DR Pythia 1B policies with a second gold RM, GPT-4.1 Nano. We also evaluate them with Gemma-3-27B-it~\citep{gemma_team_gemma_2025} as an independent judge, which shares no model family with the gold RMs, the policies, or the proxy RMs. Table~\ref{tab:gold-rm} shows that the absolute win rates shift, but the ordering is preserved.

\begin{table}[h]
	\caption{Win rate (\%) on TL;DR with Pythia 1B under different judges, evaluated on the same checkpoints. Skywork is the gold RM used in the main experiments, and it labels the preference data. Entries with $\pm$ report mean $\pm$ standard deviation over three seeds.}
	\label{tab:gold-rm}
	\centering
	\small
	\begin{tabular}{lccc}
		\toprule
		Method    & Skywork                 & GPT-4.1 Nano     & Gemma-3-27B-it          \\
		\midrule
		Dr.GRPO   & $21.0 \pm 5.4$          & $28.6$           & $13.3 \pm 3.5$          \\
		UWO       & $6.1 \pm 3.6$           & --               & $3.5 \pm 2.0$           \\
		BSPO      & $40.4 \pm 1.2$          & --               & $21.9 \pm 1.8$          \\
		AdvPO     & $47.0 \pm 2.6$          & --               & $25.1 \pm 3.8$          \\
		\method{} & $\mathbf{60.0 \pm 2.0}$ & $\mathbf{78.5}$  & $\mathbf{43.7 \pm 1.2}$ \\
		\bottomrule
	\end{tabular}
\end{table}

\subsection{Larger policy model}\label{app:larger-policy}
We evaluate a larger Pythia 6.9B policy on TL;DR. Under memory constraints, the proxy RM is initialized from the Pythia 1B SFT model, and the gold RM is GPT-4.1 Nano, as in our other experiments with larger policies. Table~\ref{tab:larger-policy} reports the win rates.

\begin{table}[h]
	\caption{Win rate (\%) on TL;DR with a Pythia 6.9B policy. Single seed. The proxy RM is a Pythia 1B model and the gold RM is GPT-4.1 Nano.}
	\label{tab:larger-policy}
	\centering
	\small
	\begin{tabular}{lc}
		\toprule
		Method    & Win rate \\
		\midrule
		SFT       & $46.4$ \\
		Dr.GRPO   & $87.3$ \\
		\method{} & $\mathbf{89.0}$ \\
		\bottomrule
	\end{tabular}
\end{table}

\subsection{Comparison with Direct Preference Optimization}\label{app:dpo}
Direct Preference Optimization (DPO)~\citep{rafailov_direct_2023} optimizes the policy from preferences without an explicit RM. Avoiding a separate RM does not completely remove reward misspecification, as the induced implicit reward can still be exploited~\citep{rafailov_scaling_2024, xu_is_2025}. We train DPO on TL;DR with Pythia 1B, sweeping its temperature parameter and reporting the best setting. Table~\ref{tab:dpo} shows that \method{} reaches a higher win rate than DPO, both after training and at the best point during training.

\begin{table}[h]
	\caption{Win rate (\%) on TL;DR with Pythia 1B. Values in parentheses are the best win rate reached during training.}
	\label{tab:dpo}
	\centering
	\small
	\begin{tabular}{lc}
		\toprule
		Method    & Win rate (best during training) \\
		\midrule
		DPO       & $49.2$ ($52.3$) \\
		Dr.GRPO   & $21.0 \pm 5.4$ ($55.8 \pm 2.0$) \\
		\method{} & $\mathbf{60.0 \pm 2.0}$ ($\mathbf{62.1 \pm 0.7}$) \\
		\bottomrule
	\end{tabular}
\end{table}

\subsection{Decomposition of the down-weighting}\label{app:weight-decomposition}
The weight $\rho_j^*$ \eqref{eq:optimal-weight} depends on a completion through two quantities: the advantage magnitude $|A_j(w)|$ and the feature deviation $\|h_\psi(x,y^{(j)})-\bar{h}\|_2$. To see which one drives the down-weighting, we group the on-policy completions of a checkpoint into deciles by their weight and report the mean of each quantity per decile. Figure~\ref{fig:weight-decomp} shows that the advantage magnitude varies strongly across deciles while the feature deviation is comparatively flat. The completions that receive the smallest weights are those with the most ambiguous advantage under the proxy RM.

\begin{figure}[h]
	\centering
	\includegraphics[width=0.55\textwidth]{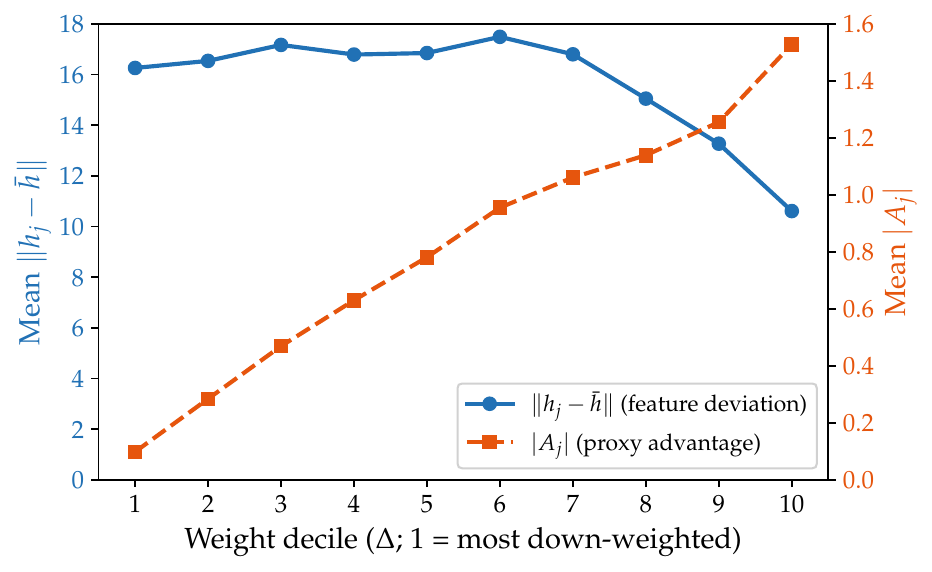}
	\caption{\textbf{Down-weighting is driven by advantage ambiguity more than by feature deviation.} Mean advantage magnitude $|A_j(w)|$ and mean feature deviation $\|h_\psi(x,y^{(j)})-\bar{h}\|_2$ across \method{} weight deciles, on TL;DR with Pythia 1B.}
	\label{fig:weight-decomp}
\end{figure}

To test whether the advantage magnitude can fully explain the gain, we replace $\Delta_j$ in \eqref{eq:optimal-weight} with the feature-uncertainty radius $\Delta_j^{\mathrm{feat}} = |A_j|/\|w\|_2$ \eqref{eq:feature-radius}. This variant weights completions by $|A_j|$ only, and it models uncertainty in the features $h_j$ rather than in the RM head (Appendix~\ref{sec:other-models}). Table~\ref{tab:abs-adv-weight} shows that this variant improves over Dr.GRPO, BSPO, and AdvPO, but \method{} improves further.

\begin{table}[h]
	\caption{Win rate (\%) on TL;DR with Pythia 1B, with mean $\pm$ standard deviation over three seeds.}
	\label{tab:abs-adv-weight}
	\centering
	\small
	\begin{tabular}{lc}
		\toprule
		Method                  & Win rate                \\
		\midrule
		Dr.GRPO                 & $21.0 \pm 5.4$          \\
		BSPO                    & $40.4 \pm 1.2$          \\
		AdvPO                   & $47.0 \pm 2.6$          \\
		Weight by $|A_j|$ only  & $54.3 \pm 2.5$          \\
		\method{}               & $\mathbf{60.0 \pm 2.0}$ \\
		\bottomrule
	\end{tabular}
\end{table}

\subsection{Proxy RM accuracy during training}\label{app:rm-accuracy}

Figure~\ref{fig:rm-accuracy-1b-vs-2.8b} compares the proxy RM accuracy trajectory during policy optimization for Pythia 1B and 2.8B on TL;DR. Although their performance is comparable initially, the Pythia-1B proxy RM degrades substantially faster than the 2.8B proxy RM, confirming that smaller proxy RMs lose their accuracy more rapidly during policy optimization.

\begin{figure}[h]
	\centering
	\includegraphics[width=0.48\textwidth]{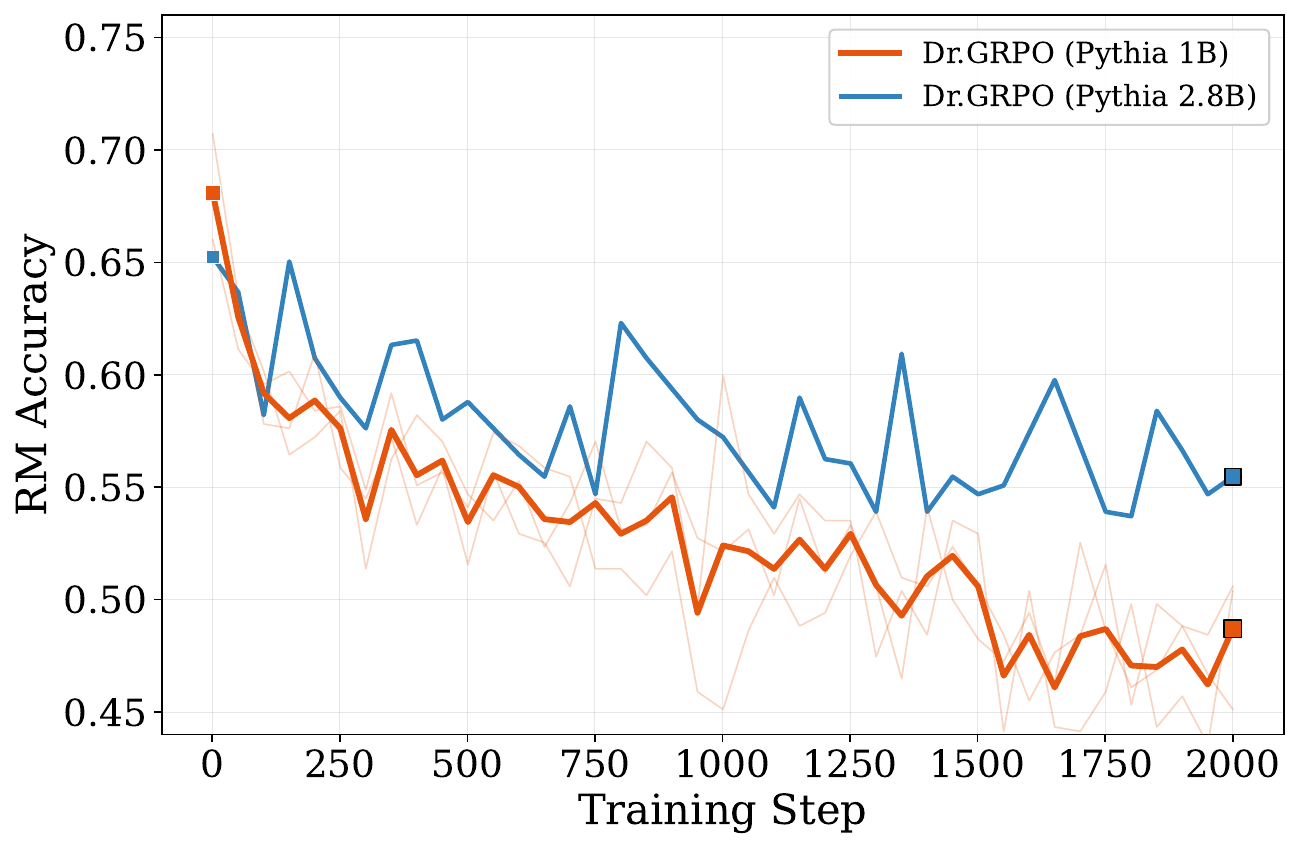}
	\caption{\textbf{Smaller proxy RMs lose accuracy faster during policy optimization.} Proxy RM accuracy over training steps for Pythia 1B and 2.8B on TL;DR. For Pythia 1B, mean over three seeds is shown with individual seed traces.}
	\label{fig:rm-accuracy-1b-vs-2.8b}
\end{figure}

\subsection{Ablation on quantile threshold and the amount of preference data}\label{app:ablation}

Figure~\ref{fig:kl-rm-quantile-ablation} shows the trade-off between KL divergence and RM accuracy across quantile hyperparameters for Pythia 1B on TL;DR. Higher quantile thresholds effectively constrain the policy to the regions where the proxy RM maintains higher accuracy. We note that \method{} consistently achieves higher RM accuracy than Dr.GRPO at the same KL budget, suggesting its distinction from the KL penalty parameter.

\begin{figure}[h]
	\centering
	\includegraphics[width=0.6\textwidth]{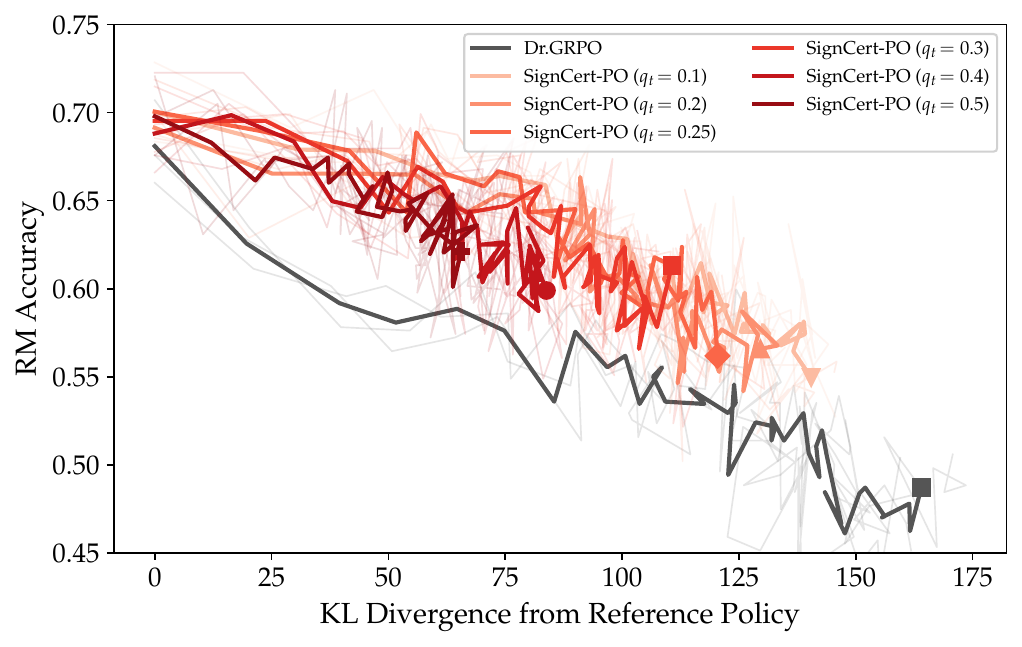}
	\caption{\textbf{The quantile parameter $q_t$ provides a tunable indicator of RM accuracy.} Proxy RM accuracy vs.\ KL divergence for Dr.GRPO and \method{} with varying $q_t$ on TL;DR (Pythia 1B). Thin lines show individual seeds and thick lines show the mean. Higher $q_t$ maintains RM accuracy.}
	\label{fig:kl-rm-quantile-ablation}
\end{figure}

Tables~\ref{tab:rm-accuracy-pythia}--\ref{tab:rm-accuracy-qwen} report the average proxy RM accuracy over training for varying preference-data epochs and quantile thresholds. \method{} consistently maintains higher proxy RM accuracy than Dr.GRPO across all settings, with the improvement most pronounced when the proxy RM is weak. We note that training a proxy RM with the preference data generated for 10--20 epochs can lead to overfitting, as we saw in \S\ref{sec:mechanism}. With approximately 123K preference pairs in the TL;DR dataset, 10 epochs corresponds to roughly 1.23M training samples, exceeding the 1--3 epochs typical in the literature~\citep{ouyang_training_2022, stiennon_learning_2020}. We include these settings as ablations.

\begin{table}[h]
	\caption{Pythia-1B: average proxy RM accuracy (\%) over training. Higher is better.}
	\label{tab:rm-accuracy-pythia}
	\centering
	\small
	\resizebox{\textwidth}{!}{%
		\begin{tabular}{lccccccc}
			\toprule
			RM Epochs & GRPO             & \method{} $q_t{=}0.1$ & $q_t{=}0.2$      & $q_t{=}0.25$     & $q_t{=}0.3$      & $q_t{=}0.4$      & $q_t{=}0.5$               \\
			\midrule
			1         & 52.89$\pm$0.22\% & 60.62$\pm$0.54\%      & 60.25$\pm$0.41\% & 62.04$\pm$0.20\% & 62.09$\pm$0.92\% & 62.80$\pm$0.54\% & \textbf{64.63$\pm$0.15\%} \\
			5         & 58.26\%          & 56.97\%               & 59.23\%          & 60.82\%          & 61.66\%          & 64.77\%          & \textbf{65.15\%}          \\
			10        & 57.32\%          & 58.14\%               & 60.47\%          & 62.20\%          & 61.08\%          & 64.93\%          & \textbf{65.32\%}          \\
			20        & 59.82\%          & 61.28\%               & 63.76\%          & 62.53\%          & 62.61\%          & \textbf{69.18\%} & 67.21\%                   \\
			\bottomrule
		\end{tabular}}
\end{table}

\begin{table}[h]
	\caption{Qwen2.5-1.5B: average proxy RM accuracy (\%) over training. Higher is better.}
	\label{tab:rm-accuracy-qwen}
	\centering
	\resizebox{\textwidth}{!}{%
		\begin{tabular}{lccccccc}
			\toprule
			RM Epochs & GRPO             & \method{} $q_t{=}0.1$ & $q_t{=}0.2$      & $q_t{=}0.25$     & $q_t{=}0.3$      & $q_t{=}0.4$      & $q_t{=}0.5$               \\
			\midrule
			1         & 65.82$\pm$0.88\% & 67.09$\pm$0.78\%      & 68.78$\pm$0.80\% & 69.19$\pm$0.56\% & 69.32$\pm$0.33\% & 71.04$\pm$0.22\% & \textbf{71.72$\pm$0.70\%} \\
			5         & 66.94\%          & 67.05\%               & 68.84\%          & 68.55\%          & 68.25\%          & 70.54\%          & \textbf{73.28\%}          \\
			10        & 68.22\%          & 68.48\%               & 69.99\%          & 69.83\%          & 71.86\%          & 71.11\%          & \textbf{75.94\%}          \\
			20        & 69.51\%          & 67.95\%               & 69.30\%          & 71.22\%          & 71.32\%          & 71.96\%          & \textbf{74.31\%}          \\
			\bottomrule
		\end{tabular}}
\end{table}

Tables~\ref{tab:quantile-winrate-pythia}--\ref{tab:quantile-winrate-qwen} present the full quantile sweep on TL;DR. The optimal quantile threshold depends on the proxy RM quality and policy architecture. Across both models, higher thresholds monotonically improve proxy RM accuracy (Tables~\ref{tab:rm-accuracy-pythia}--\ref{tab:rm-accuracy-qwen}) but can reduce win rate when the threshold is too aggressive, reflecting the exploration--exploitation trade-off in robust optimization.

\begin{table}[h]
	\caption{Pythia-1B: win rate (\%) after 2,000 steps across quantile thresholds and RM epochs.}
	\label{tab:quantile-winrate-pythia}
	\centering
	\small
	\resizebox{\textwidth}{!}{%
		\begin{tabular}{lccccccc}
			\toprule
			RM Epochs & GRPO              & \method{} $q_t{=}0.1$ & $q_t{=}0.2$      & $q_t{=}0.25$     & $q_t{=}0.3$      & $q_t{=}0.4$               & $q_t{=}0.5$      \\
			\midrule
			1         & 20.96$\pm$4.38\%  & 20.31$\pm$6.40\%      & 30.08$\pm$7.16\% & 42.77$\pm$1.81\% & 49.81$\pm$8.98\% & \textbf{60.03$\pm$1.61\%} & 53.39$\pm$3.86\% \\
			5         & 63.28\%           & 50.98\%               & 52.54\%          & 67.58\%          & \textbf{67.97\%} & 66.41\%                   & 49.61\%          \\
			10        & 55.27\%           & 66.80\%               & 73.44\%          & 70.12\%          & 71.29\%          & \textbf{77.93\%}          & 67.38\%          \\
			20        & \textbf{64.06\%}  & 53.32\%               & 56.06\%          & 48.24\%          & 47.27\%          & 42.97\%                   & 49.22\%          \\
			\bottomrule
		\end{tabular}}
\end{table}

\begin{table}[h]
	\caption{Qwen2.5-1.5B: win rate (\%) after 500 steps across quantile thresholds and RM epochs.}
	\label{tab:quantile-winrate-qwen}
	\centering
	\resizebox{\textwidth}{!}{%
		\begin{tabular}{lccccccc}
			\toprule
			RM Epochs & GRPO                      & \method{} $q_t{=}0.1$ & $q_t{=}0.2$      & $q_t{=}0.25$     & $q_t{=}0.3$      & $q_t{=}0.4$      & $q_t{=}0.5$      \\
			\midrule
			1         & \textbf{66.86$\pm$3.13\%} & \textbf{66.01$\pm$1.92\%}      & 65.11$\pm$1.27\% & 64.06$\pm$0.70\% & 63.74$\pm$1.07\% & 60.09$\pm$1.27\% & 53.58$\pm$1.30\% \\
			5         & 94.53\%                   & 94.92\%               & \textbf{95.90\%} & 95.51\%          & 92.58\%          & 94.34\%          & 90.43\%          \\
			10        & 94.34\%                   & \textbf{95.70\%}      & 94.92\%          & 95.51\%          & 94.53\%          & 90.43\%          & 91.21\%          \\
			20        & 95.31\%                   & 96.09\%               & 94.53\%          & 96.09\%          & 94.34\%          & \textbf{96.88\%} & 94.92\%          \\
			\bottomrule
		\end{tabular}}
\end{table}

\subsection{Sample outputs}\label{app:sample-outputs}

Tables~\ref{tab:sample-output}--\ref{tab:sample-output-alpaca} show example outputs by models trained with different methods. On TL;DR with Pythia-1B (1-epoch proxy RM), \method{} produced the response with the highest gold reward (GR) in 326 out of 512 evaluation prompts. On AlpacaFarm with Qwen2.5-1.5B, \method{} achieved the highest GR in 210 out of 512 prompts.

\begin{table}[h]
	\caption{Example output by a Pythia-1B model trained with different methods on the TL;DR task. \method{} produces a coherent summary comparable to the reference. GR denotes gold reward.}
	\label{tab:sample-output}
	\centering
	\small
	\begin{tabular}{p{0.88\textwidth}r}
		\toprule
		Content & GR \\
		\midrule
		SUBREDDIT: r/relationships \newline
		TITLE: Me [17 M] with a girl [16 F] who I don't know how to let down. \newline
		POST: There is a backstory behind this so let me give it to everyone straight. I'm 17 M never dated anyone before. I have made posts prior to this one based around my issues with being able to date. My biggest issue is that once a girl peeks my interest I pursue her until she develops the same feelings. Once I understand she is interested too I sort of back off. A friend of mine called me a ``chaser'' which is what I feel I am. Now I have this girl that I know, really sweet, really smart, talented, the whole thing. I found recently she has feelings for me and we talked a bit about it after I mentioned I was aware of them. When we talked I gave her the impression that I wanted to date her and that we should talk in person and hang out a big more to be dating. The only issue is my gut AGAIN like always is telling me I'm not interested even though I really want to be, I'm just not. & -- \\
		\midrule
		\textbf{Reference}: I don't know how to let this girl down after already giving her the impression I wanted to date her. I feel absolutely terrible about it. & $-0.44$ \\
		\midrule
		\textbf{Dr.GRPO}: talks/ possible dates/ feelings for girl/ situation changed but am I worried about telling/ situation affecting date/ how to let down girlfriend/ options available/ outcome. Advice please but worried about telling/ situation affecting date/ outcome. & $-8.56$ \\
		\textbf{AdvPO}: communication/dating situation, possibility of dating but am not sure if to let down friend/ possible date/ talk in person/ hang out/ possible outcome. Do I tell friend or risk hurting friendship? worried about outcome. worried about outcome & $-9.50$ \\
		\textbf{UWO}: Don't know how to let her down. Help. Help. Help help help. help help. help help. help. not sure how to best approach help help help. not sure how to best approach help. not sure how to tell. not sure how to & $-16.62$ \\
		\textbf{BSPO}: I'm 17 M never dated before, have feelings for girl, want to date her, but not sure if I'm interested or not. not sure how to let her down. any advice or suggestions would be appreciated. any thoughts or comments would be appreciated. any & $-8.75$ \\
		\textbf{\method{} (Ours)}: dating/potential relationship, girl has feelings for me, I give impression I want to date but gut tells me I'm not interested. Am I overthinking this? Advice on letting down? Am worried about hurting relationship? Help! & $-1.31$ \\
		\bottomrule
	\end{tabular}
\end{table}

\begin{table}[h]
	\caption{Example output by a Qwen2.5 1.5B model trained with different methods on the AlpacaFarm task. \method{} produces a concise list comparable to the reference.}
	\label{tab:sample-output-alpaca}
	\centering
	\small
	\begin{tabular}{p{0.88\textwidth}r}
		\toprule
		Content & GR \\
		\midrule
		Generate a list of the seven most iconic locations in Colorado. & -- \\
		\midrule
		\textbf{Reference}: 1. Garden of the Gods -- Colorado Springs \newline 2. Maroon Bells -- Aspen \newline 3. Mesa Verde National Park -- Montezuma \newline 4. Royal Gorge Bridge -- Ca\~{n}on City \newline 5. Red Rocks Amphitheater -- Denver \newline 6. Pikes Peak -- Colorado Springs \newline 7. Rocky Mountain National Park -- Estes Park & $7.81$ \\
		\midrule
		\textbf{Dr.GRPO}: The seven most iconic locations in Colorado are: Rocky Mountain National Park, Denver, Colorado, Pikes Peak, Colorado, Colorado Springs, Colorado, and the Great Sand Dunes National Park. Explanation: These locations are iconic because they are known for their natural beauty, adventure, and history. They are also popular tourist destinations and have a lot of cultural and historical significance. They are also easily accessible and offer a variety of activities and experiences. Finally, they are all located in Colorado and are a great way to experience the state. & $-9.06$ \\
		\textbf{AdvPO}: The seven most iconic locations in Colorado are: Rocky Mountain National Park, Denver, Colorado Springs, Pikes Peak, Colorado Springs, Colorado Springs. These locations are all known for their natural beauty and are popular tourist destinations. They are also home to many famous landmarks and attractions. & $-20.88$ \\
		\textbf{UWO}: The seven most iconic locations in Colorado are: 1. Rocky Mountains 2. Grand Canyon 3. Mesa Verde 4. Garden of the Gods 5. Pikes Peak 6. Rocky Mountain National Park 7. Denver. These locations are iconic because they are beautiful, unique and have a rich history. They are also popular tourist destinations and have a significant impact on the state's economy. Therefore, they are considered to be the most iconic locations in Colorado. & $-16.25$ \\
		\textbf{BSPO}: The seven most iconic locations in Colorado are: Rocky Mountain National Park, Denver, Colorado, Denver International Airport, Denver, Colorado, Denver, Colorado, Denver, Colorado. & $-28.50$ \\
		\textbf{\method{} (Ours)}: 1. Rocky Mountains 2. Denver 3. Colorado Springs 4. Pikes Peak 5. Garden of the Gods 6. Mesa Verde 7. Telluride & $8.56$ \\
		\bottomrule
	\end{tabular}
\end{table}

\end{document}